%% file: manuscript_report.tex
\documentclass[a4paper, 10pt, conference]{ieeeconf}      

\IEEEoverridecommandlockouts                              

\input{packages}
\input{commands}

\input{notation}

\title{\LARGE \bf
Feedback Motion Prediction for Safe Unicycle Robot Navigation \\(Technical Report)
}

\author{Aykut \.{I}\c{s}leyen and Nathan van de Wouw and \"{O}m\"{u}r Arslan
\thanks{The authors are with the Department of Mechanical Engineering, Eindhoven University of Technology, P.O. Box 513, 5600 MB Eindhoven, The Netherlands. The authors are also affiliated with the Eindhoven AI Systems Institute. Emails:  \{a.isleyen, n.v.d.wouw, o.arslan\}@tue.nl}%
}

\begin{document}

\maketitle
\thispagestyle{empty}
\pagestyle{empty}

\begin{abstract}
As a simple and robust mobile robot base, differential drive robots that can be modelled as a kinematic unicycle find significant applications in logistics and service robotics in both industrial and domestic settings.
Safe robot navigation around obstacles is an essential skill for such unicycle robots to perform diverse useful tasks in complex cluttered environments, especially around people and other robots.
Fast and accurate safety assessment plays a key role in reactive and safe robot motion design.
In this paper, as a more accurate and still simple alternative to the standard circular Lyapunov level sets, we introduce novel conic feedback motion prediction methods for bounding the close-loop motion trajectory of the kinematic unicycle robot model under a standard unicycle motion control approach.
We present an application of unicycle feedback motion prediction for safe robot navigation around obstacles using reference governors, where the safety of a unicycle robot is continuously monitored based on the predicted future robot motion. 
We investigate the role of motion prediction on robot behaviour in numerical simulations and conclude that fast and accurate feedback motion prediction is key for fast, reactive, and safe robot navigation around obstacles.
\end{abstract}

\section{Introduction}
\label{sec.Introduction}

Mobile robots play a key role in industrial (e.g., warehouse robots in logistics \cite{fiorini_botturi_ISR2008}) and domestic (e.g., service robots for household \cite{jones_RAM2006}) automation.
Due to their simplicity, high maneuverability, and ease of control and maintenance, differential drive robots that can be modelled as a simple kinematic unicycle become a standard choice as a mobile robot base for many such application settings \cite{thai_etal_IJMER2022}.
Safe, smooth, and fast navigation around obstacles is a crucial requirement for such unicycle robots to perform various time-critical tasks in complex environments, especially around people \cite{philippsen_siegwart_ICRA2003, prassler_scholz_fiorini_RAM2001}  and other mobile robots \cite{snape_etal_IROS2010}. 
Motion prediction plays a key role in safe and smooth mobile robot motion design  \cite{philippsen_siegwart_ICRA2003, chakravarthy_debasish_TSM1998, fox_burgard_thrun_RAM1997, fiorini_shiller_IJRR1998}. 

In this paper, we introduce a new family of conic unicycle feedback motion prediction methods (see \reffig{fig.motion_prediction_demo}) that offers an accurate and computationally efficient tool for bounding the closed-loop motion trajectory of a unicycle robot under a standard unicycle motion control approach towards a given goal position.
We show that such unicycle feedback motion prediction methods can be effectively used for the fast and accurate safety assessment of robot motion around obstacles and so for fast, reactive, and safe robot navigation.

\begin{figure}[t]
\centering
\begin{tabular}{@{}c @{\hspace{0.5mm}} c @{\hspace{0.5mm}} c@{}}
\includegraphics[width = 0.33\columnwidth]{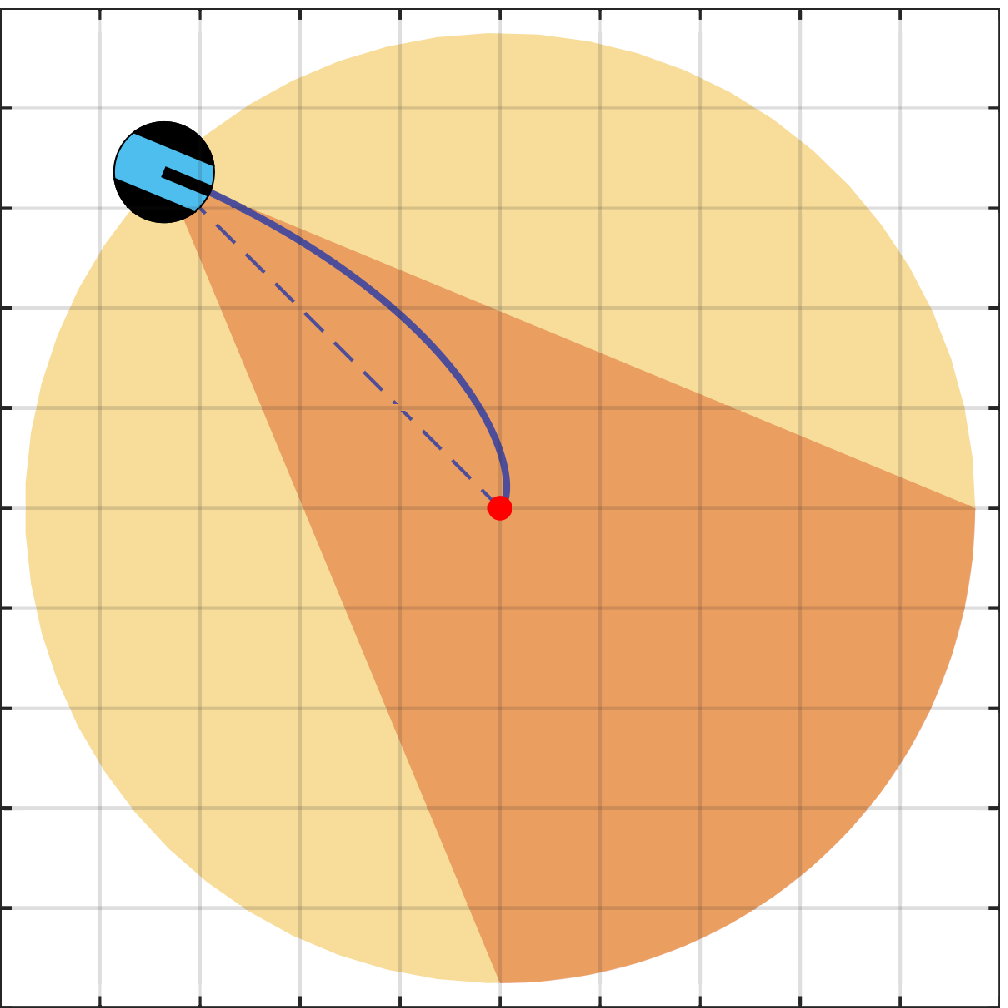} & 
\includegraphics[width = 0.33\columnwidth]{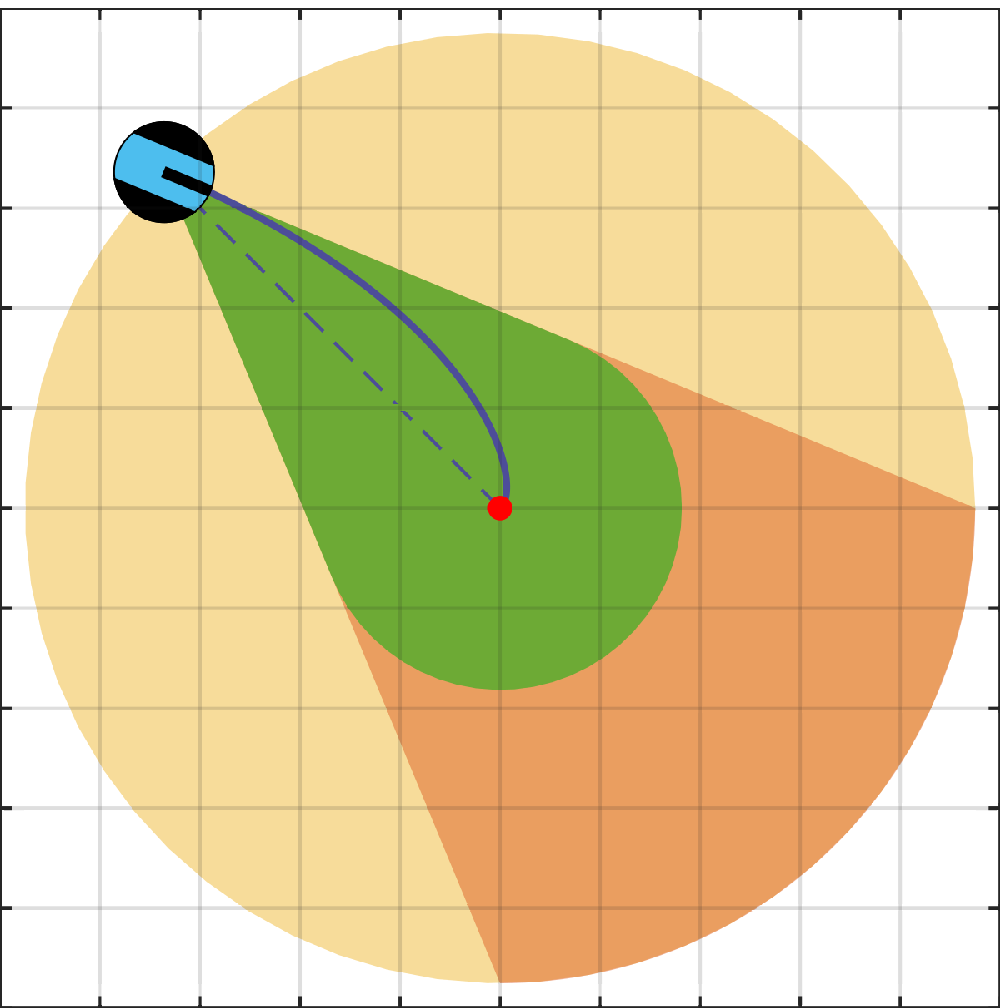} & 
\includegraphics[width = 0.33\columnwidth]{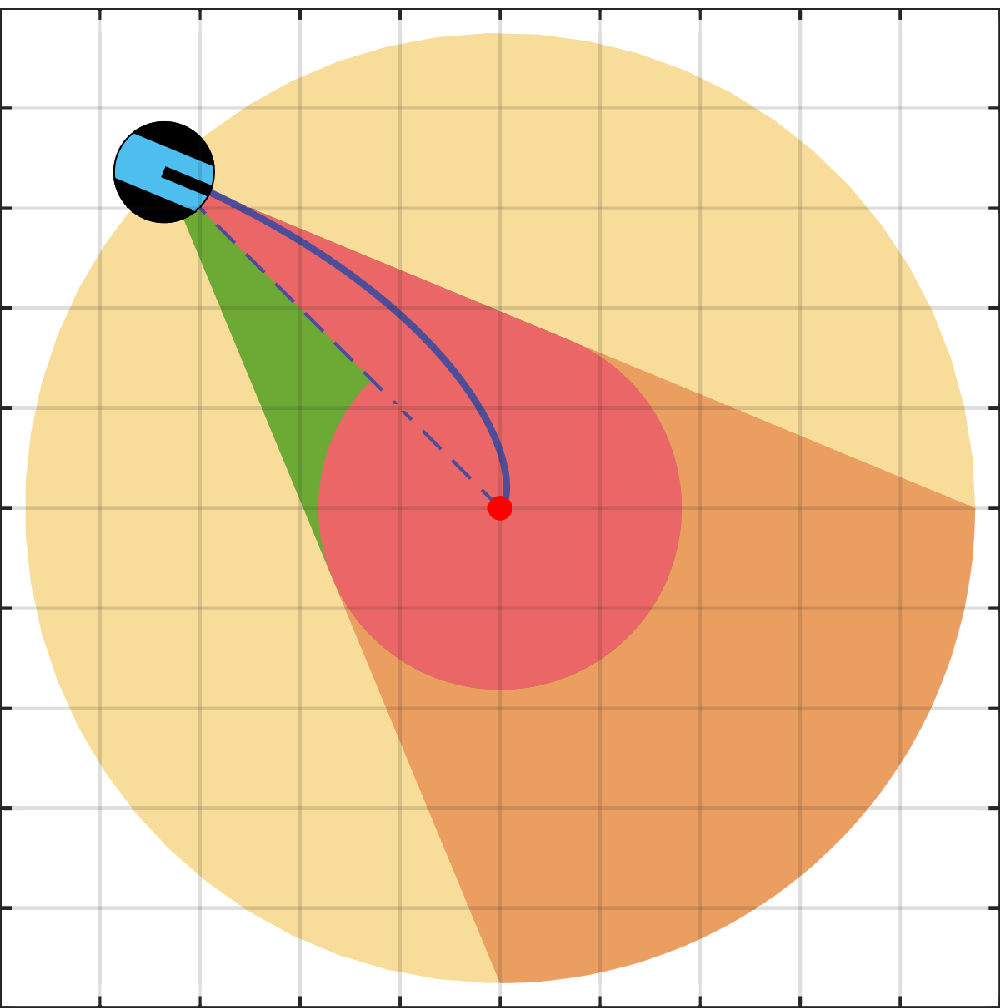}
\end{tabular}
\caption{Unicycle feedback motion prediction that bounds the close-loop unicycle motion trajectory (blue line) towards a given goal position (red point): Lyapunov motion ball (yellow), bounded motion cone (orange), ice-cream motion cone (green), and truncated ice-cream motion cone (red).}
\label{fig.motion_prediction_demo}
\end{figure}

\subsection{Motivation and Relevant Literature}

Nonholonomic motion planning for mobile robots is computationally hard \cite{laumond_MotionPlanning1998}.
The classical approach for safe mobile robot navigation follows a two-step approach: first, construct a collision-free reference path or trajectory, and then execute it via path following or trajectory tracking control until faced with a collision \cite{li_zhang_etal_ROBIO2017}.
However, such an uncoupled planning and control approach often suffers from significant replanning cycles in practice due to its open-loop nature \cite{brezak_petrovic_IFAC2011}.
Integrated motion planning and control by sequential composition \cite{burridge_rizzi_koditschek_IJRR1999} of positively invariant (e.g., Lyapunov level or funnel) sets \cite{blanchini_Automatica1999} offers a robust and adaptive solution for safe and smooth feedback motion planning \cite{pathak_agrawal_TRO2005, conner_choset_rizzi_RSS2006, majumdar_tedrake_IJRR2017, danielson_berntorp_cairano_weiss_ACC2020}.
However, such integrated planning and control approaches are usually computationally inefficient and conservative because finding collision-free invariant sets that cover the entire collision-free configuration space of a robotic system in an arbitrary environment is a challenge \cite{conner_choset_rizzi_RSS2006}.   
In this paper, as an alternative to conservative Lyapunov level sets, we construct new analytic conic unicycle feedback motion prediction methods that can accurately bound the closed-loop unicycle robot motion under a standard unicycle control approach.
We apply unicycle feedback motion prediction for provably correct and safe reference motion following (without replanning) by establishing a continuous bidirectional safety interface between high-level motion planning and low-level motion control based on a reference governor and the safety assessment of predicted future robot motion \cite{isleyen_arslan_RAL2022}.

Motion prediction of anticipating the future motion of an autonomous agent plays a key role in the safety assessment, control, and planning of autonomous robots around obstacles \cite{lefevre_vasquez_laugier_ROBOMECH2014}.
Most existing motion prediction algorithms use simple physical motion (e.g., constant velocity, acceleration, and turning rate) models  \cite{schubert_richter_wanielik_ICIF2008} or pre-defined/learned motion patterns (a.k.a. motion primitives and maneuvers) \cite{schreier_willert_adamy_TITS2016, bennewitz_etal_IJRR2005} to estimate future system behaviour by either running the open-loop forward system simulation or performing high-level motion planning.
To close the gap between motion prediction and motion control, we consider feedback motion prediction that aims at finding a motion set that contains the closed-loop motion trajectory of a dynamical system under a specific control policy \cite{arslan_isleyen_arXiv2023, isleyen_arslan_RAL2022}.
Reachability analysis offers advanced computational tools for estimating such motion sets for complex dynamical control systems associated with some admissible sets of initial/goal states and control inputs \cite{althoff_dolan_TR02014, althoff_frehse_girard_ARCRAS2021}, but often comes with a high computational cost which limits their applications to real-time, reactive, and safe robot motion planning and control since mobile robot platforms usually come with limited computational resources.
For globally asymptotically stable autonomous dynamical systems (with fixed deterministic state-feedback control and without any noise and disturbance), the notion of forward and backward reachable sets \cite{mitchell_HSCC2007} is trivial because the forward reachable set corresponds to the system trajectory due to the autonomous nature of the system dynamics whereas the backward reachability set is the entire state space due to the global stability.
In this paper, we use the forward simulation of the stable closed-loop unicycle dynamics as the baseline ground-truth motion prediction method.
In our numerical simulations, we demonstrate the effectiveness of the proposed conic feedback unicycle motion prediction methods for accurately capturing the closed-loop unicycle motion compared to the forward system simulation and Lyapunov motion prediction.  

Reference governors are add-on constrained control approaches for pre-stabilized dynamical systems to follow a given reference motion by minimally modifying the reference motion such that the expected closed-loop system motion satisfies system constraints at all times \cite{bemporad_TAC1998, gilbert_kolmanovsky_Automatica2002, garone_nicotra_TAC2015}.
The separation of stability and constraint satisfaction allows for systematically applying standard stabilizing robot control methods to complex constrained application settings. 
In robotics, reference governors are applied for safe robot navigation to separately address global navigation, stability, and safety requirements at different stages by high-level planning and low-level control \cite{arslan_koditschek_ICRA2017}.
Reference governors are successfully demonstrated for safe navigation of fully actuated higher-order robot systems using Lyapunov invariance sets \cite{arslan_koditschek_ICRA2017, li_ICRA2020, li_2020, isleyen_arslan_RAL2022}.
In this paper, we demonstrate the application of reference governors for safe unicycle robot navigation around obstacles with nonholonomic constraints by using unicycle feedback motion prediction.  
We also systematically investigate the role of different unicycle motion prediction methods on the governed robot motion in numerical simulations.

\subsection{Contributions and Organization of the Paper}

This paper introduces a family of novel conic feedback motion prediction methods for the kinematic unicycle robot model to bound the closed-loop unicycle motion trajectory under a standard forward motion control approach towards a goal position in \refsec{sec.UnicycleDynamicsControlPrediction}.
The proposed conic motion prediction methods are more accurate than the standard Lyapunov level sets (see \reffig{fig.motion_prediction_demo}) and are still easy to represent and compute.
In \refsec{sec.SafeNavigation}, we present an application of these unicycle feedback motion prediction methods for safe robot navigation using a reference governor, where the safety of the unicycle motion is continuously monitored using the collision distance of the predicted robot motion.
In \refsec{sec.NumericalSimulations}, we provide numerical simulations to demonstrate the effectiveness of the proposed conic feedback motion prediction methods compared to Lyapunov level sets and forward system simulation.
We conclude in \refsec{sec.Conclusions} with a summary of our contributions and future directions.

\section{Unicycle Dynamics, Control, \& Prediction}
\label{sec.UnicycleDynamicsControlPrediction}

In this section, we briefly describe the kinematic unicycle robot model and present a unicycle forward motion controller to navigate towards a given goal position. 
Then, we provide several feedback motion prediction methods that bound the unicycle motion trajectory under the forward motion control.  

\subsection{Kinematic Unicycle Robot Model}
\label{sec.KinematicUnicycleRobotModel}

In the Euclidean plane $ \R^2$, we consider a kinematic unicycle robot whose pose (a.k.a., state and configuration) is represented by its position $\pos \in \R^2$ and forward orientation angle $\ort \in [ -\pi, \pi )$ that is measured in radians counterclockwise from the horizontal axis.
The equations of motion of the kinematic unicycle robot model are given by
\begin{align} \label{eq.UnicycleDynamics}
\dot{\pos} = \linvel \ovect{\ort} \quad \text{and} \quad \dot{\ort} = \angvel 
\end{align}
where  $\linvel \in \R$ and $\angvel \in \R$ are the scalar control inputs, respectively, specifying the linear and angular velocity of the unicycle robot.
Hence, by definition, the unicycle robot model is underactuated (i.e., three state variables, but only two control inputs) and has the nonholonomic motion constraint of no sideway motion, i.e., $\nvecTsmall{\ort}\dot{\pos} = 0$.  
We also assume that the unicycle robot is only allowed to go forward (i.e., $\linvel \geq 0$), because, for example, its field of sensing might be restricted to the front direction or it might have a manipulator arm in the front.

\subsection{Unicycle Forward Motion Control}
\label{sec.UnicycleForwardMotionControl}

Based on a standard, globally asymptotically stable unicycle control approach \cite{astolfi_JDSMC1999, arslan_koditschek_IJRR2019}, we construct a unicycle forward motion controller, denoted by $\ctrl_{\goal}(\pos, \ort) = (\linvel_{\goal}(\pos, \ort), \angvel_{\goal}(\pos, \ort))$, that moves the unicycle robot in the forward direction towards any given goal position $\goal \in \R^2$ by determining the linear and angular velocity inputs as 
\begin{subequations} \label{eq.forward_unicycle_control}
\begin{align}
\linvel_{\goal}(\pos, \ort) &= \lingain \max \plist{\! 0 , \ovecTsmall{\ort}\! \!(\goal-\pos)\!\!}
\\
\angvel_{\goal}(\pos, \ort) &= \anggain \atantwo\plist{\!\nvecTsmall{\ort}\!\! (\goal-\pos), \ovecTsmall{\ort}\!\! (\goal-\pos)\!\!}\!\!
\end{align}
\end{subequations}
where  $\lingain > 0$ and $\anggain >0$ are positive scalar control gains for the linear and angular velocity, respectively, and  $\atantwo(y,x)$ is the 2-argument arctangent function that returns the counterclockwise angle (in radians in $[-\pi, \pi)$) from the horizontal axis to the ray starting from the origin to the point $(x,y)$ in the Euclidean plane.\reffn{fn.UnicycleControlContinuity}

As a globally asymptotically stable controller, the unicycle forward motion control in \refeq{eq.forward_unicycle_control} decreases both the Euclidean distance and the perpendicular alignment distance to the goal as well as orients the unicycle robot towards the goal in finite time and then maintains a persistent goal alignment, which is formally stated below and essential for unicycle feedback motion prediction later in \refsec{sec.UnicycleFeedbackMotionPrediction}. 

\addtocounter{footnote}{1}\footnotetext{\label{fn.UnicycleControlContinuity}Note that we set $\angvel = 0$ when the robot is at the goal (i.e., $\pos = \goal$) to resolve the indeterminacy in the angular velocity since $\atantwo(0,0)$ is undefined \cite{astolfi_JDSMC1999}. This naturally introduces a discontinuity in control at the goal position as necessitated by Brockett's theorem \cite{Brockett_DGCT1983}. Otherwise, the unicycle forward motion control in \refeq{eq.forward_unicycle_control} is Lipschitz continuous almost everywhere, away from the goal position for any unicycle pose $(\pos, \ort) \in \R^{2} \times [-\pi, \pi)$ that satisfies $\ovecTsmall{\ort}(\goal - \pos) \neq - \norm{\goal - \pos}$.}

\begin{lemma} \label{lem.GlobalStability}
\emph{(Global Stability)}
The unicycle forward motion control $\ctrl_{\goal}$ in \refeq{eq.forward_unicycle_control} asymptotically brings all unicycle poses $(\pos, \ort)$ in $\R^{2}\times [-\pi, \pi)$ to any given goal position $\goal \in \R^2$, i.e., the closed-loop unicycle trajectory $(\pos(t), \ort(t))$ satisfies
\begin{align}
\lim_{t\rightarrow \infty} \pos(t) = \goal.
\end{align}   
\end{lemma}
\begin{proof}
See \refapp{app.GlobalStability}.
\end{proof}

\begin{lemma} \label{lem.EuclideanDistance2Goal}
\emph{(Euclidean Distance to Goal)} 
Under the unicycle forward motion control $\ctrl_{\goal}$ in \refeq{eq.forward_unicycle_control}, the Euclidean distance $\norm{\pos - \goal}$ of any unicycle pose $(\pos,\ort) \in \R^{2} \times [-\pi, \pi)$ to any given goal position $\goal \in \R^{2}$ is decreasing over time, i.e.,
\begin{align}
\frac{\diff}{\diff t} \norm{\pos - \goal}^2 = - 2 \lingain \plist{\!\ovecTsmall{\ort}\! \!(\goal - \pos)\!}^{\!\!2} \leq 0.
\end{align} 
\end{lemma}
\begin{proof}
See \refapp{app.EuclideanDistance2Goal}.
\end{proof}

\begin{lemma} \label{lem.FiniteTimeGoalAlignment}
\emph{(Finite-Time Goal Alignment)}
The unicycle forward motion control  $\ctrl_{\goal}$ in \refeq{eq.forward_unicycle_control} adjusts the unicycle orientation towards any given goal position $\goal \in \R^{2}$ in at most  $\frac{1}{\anggain}$ seconds, where $\anggain > 0$ is the angular velocity gain, that is to say, the unicycle pose trajectory $(\pos(t), \ort(t))$ starting at $t = 0$ from any initial pose $(\pos_0, \ort_0) \in \R^{2} \times [-\pi, \pi)$ away from the goal (i.e., $\pos_0 \neq \goal$) satisfies
\begin{align}
\ovecTsmall{\ort(\tfrac{1}{\anggain})}(\goal - \pos(\tfrac{1}{\anggain})) > 0.
\end{align}
\end{lemma} 
\begin{proof}
See \refapp{app.FiniteTimeGoalAlignment}.
\end{proof}

\begin{lemma}\label{lem.PersistentGoalAlignment}
\emph{(Persistent Goal Alignment)}
For any initial unicycle pose $(\pos_0, \ort_0) \in \R^{2} \times [-\pi, \pi)$ at $t = 0$ that points towards a given goal position $\goal \in \R^{2}$, the unicycle forward motion control  \refeq{eq.forward_unicycle_control} keeps the unicycle pose $(\pos(t), \ort(t))$ aligned towards the goal for all future time $t \geq 0$, i.e.,
\begin{align}
\ovecTsmall{\ort_0} \! (\goal - \pos_0) \geq 0 \Longrightarrow \ovecTsmall{\ort(t)} \! (\goal - \pos(t)\!) \geq 0, 
\end{align} 
where the inequalities are strict  for $\pos_0 \neq \goal$ and $\pos(t) \neq \goal$.
\end{lemma}
\begin{proof}
See \refapp{app.PersistentGoalAlignment}.
\end{proof}

\begin{lemma} \label{lem.PerperdicularGoalAlignmentDistance}
\emph{(Perpendicular Goal Alignment Distance)}
For any unicycle pose $(\pos, \ort)  \in \R^2 \times [-\pi, \pi)$ that points towards the goal position $\goal \in \R^2$,  the perpendicular goal alignment distance $\dist_{\goal}(\pos,\ort)$ that is defined as
\begin{align}\label{eq.PerpendicularAlignmentDistance}
\dist_{\goal}(\pos, \ort) &:= \absval{\nvecTsmall{\ort}\!(\goal - \pos)}
\end{align}
is decreasing under the unicycle forward control in \refeq{eq.forward_unicycle_control}, i.e.,
\begin{align}
\ovecTsmall{\ort} \!\! (\goal - \pos) \geq 0  \Longrightarrow \frac{\diff }{\diff t} \dist_{\goal}(\pos, \ort) \leq 0.
\end{align}
\end{lemma}
\begin{proof}
See \refapp{app.PerpendicularGoalAlignmentDistance}.
\end{proof}

\begin{figure*}[t]
\centering
\begin{tabular}{@{}c@{\hspace{1.5mm}}c@{\hspace{2mm}}c@{\hspace{1.5mm}}c@{}}
\includegraphics[width = 0.23\textwidth]{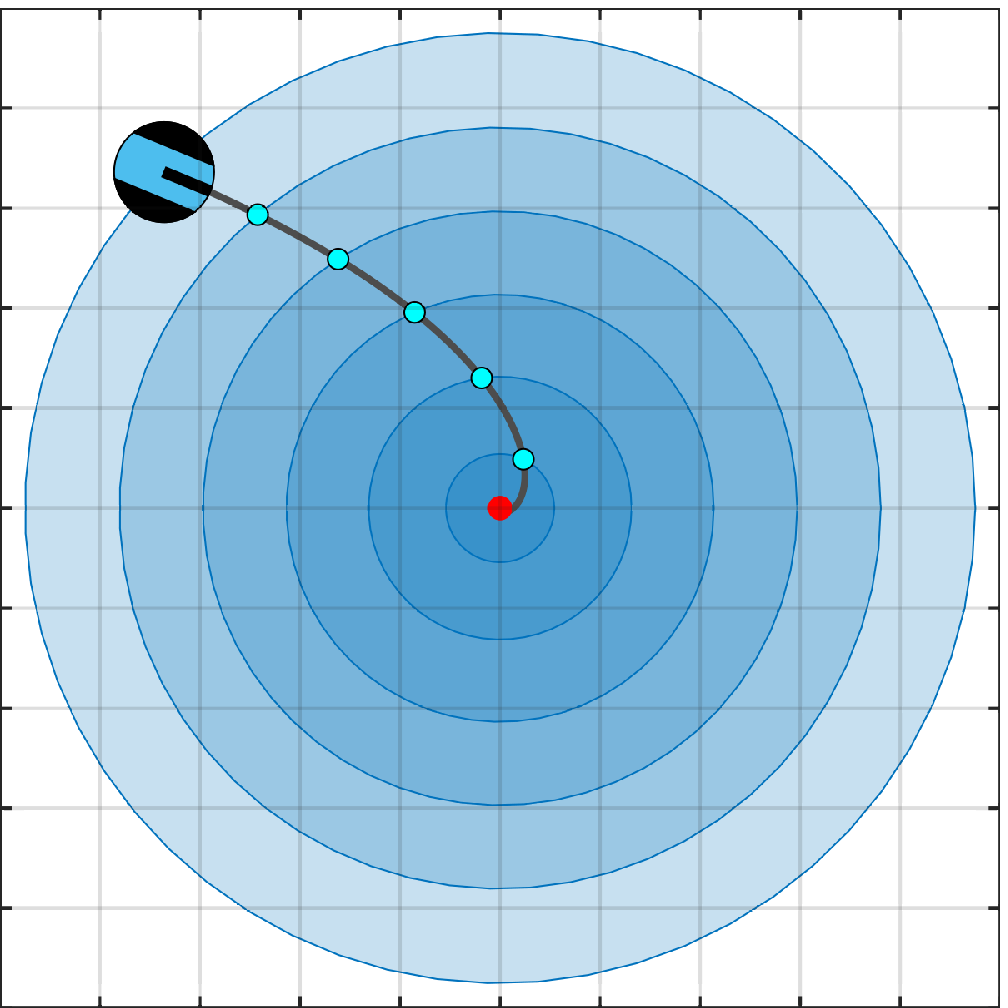} & 
\includegraphics[width = 0.23\textwidth]{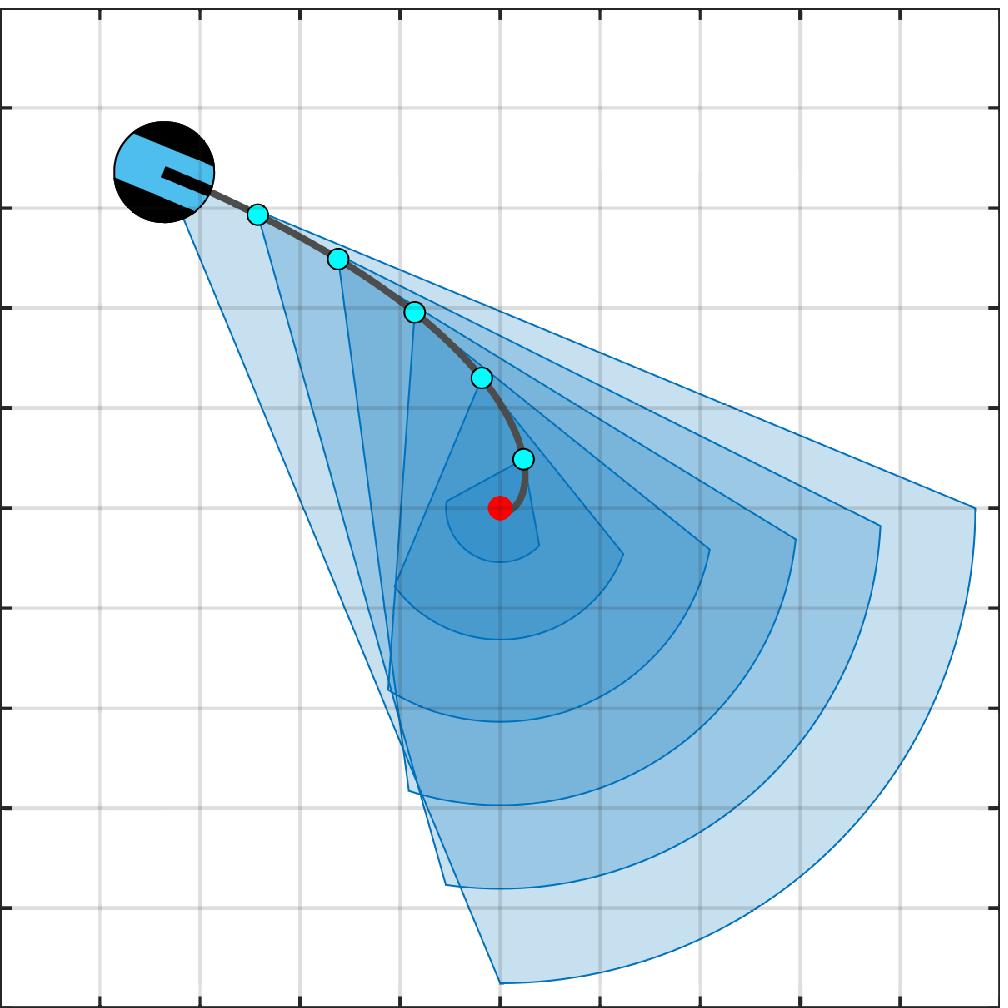} &
\includegraphics[width = 0.23\textwidth]{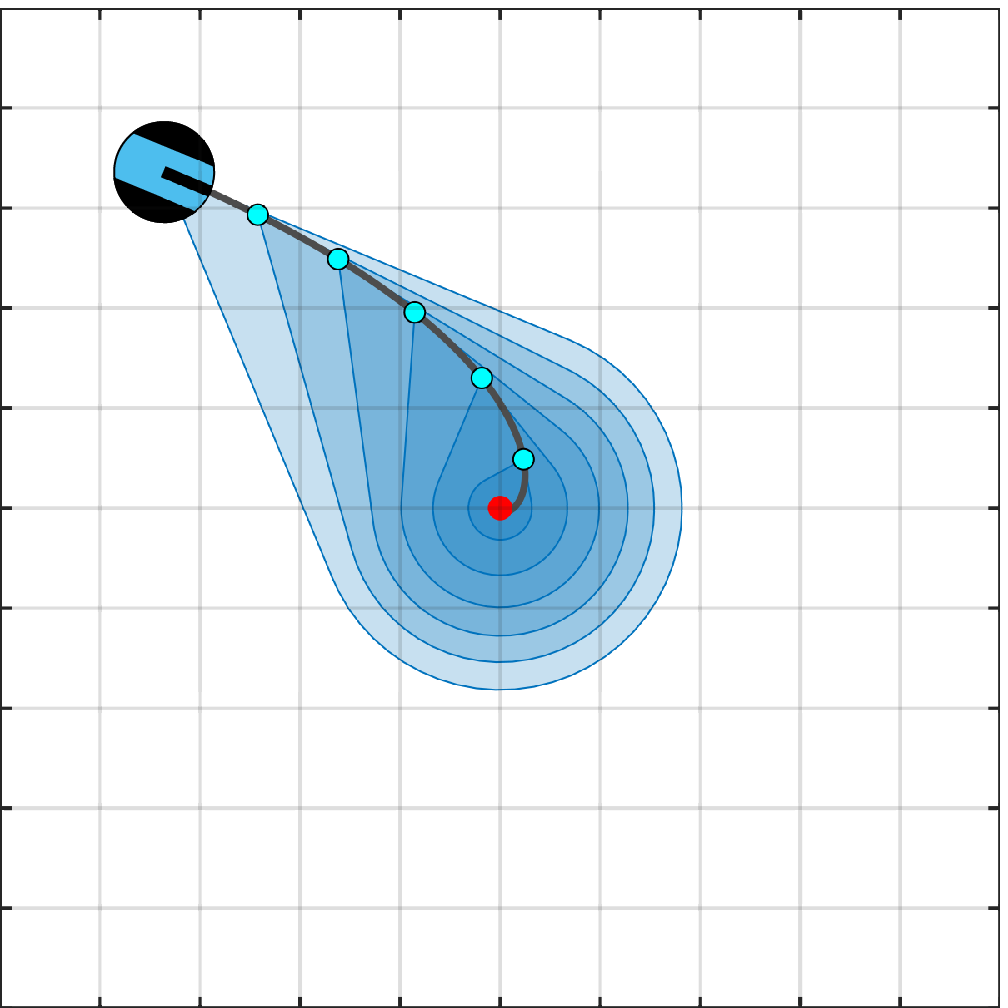} & 
\includegraphics[width = 0.23\textwidth]{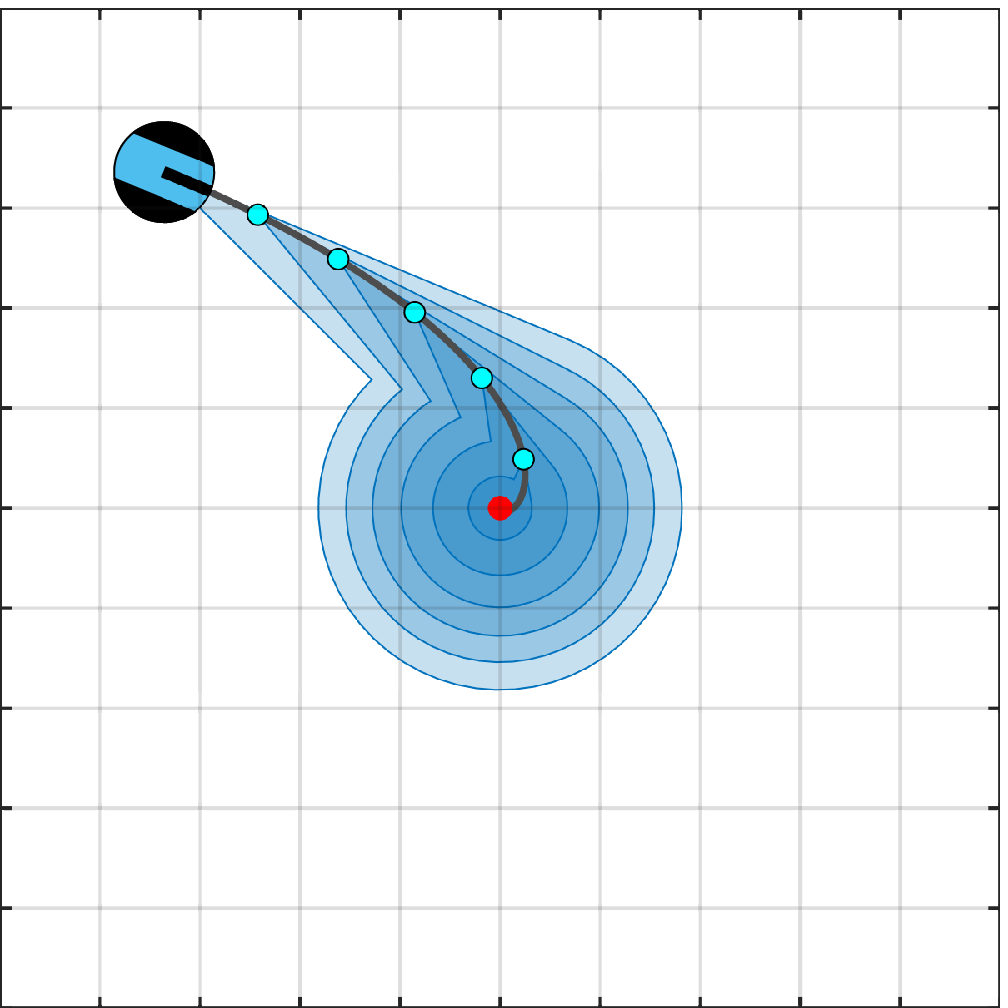}
\\[-1mm]
\footnotesize{(a)} & \footnotesize{(b)} & \footnotesize{(c)} & \footnotesize{(d)}
\end{tabular}
\vspace{-2mm}
\caption{Positive inclusion of unicycle feedback motion predictions: (a) Circular motion prediction, (b) Bounded conic motion prediction, (c) Ice-cream motion cone,  (d) Truncated ice-cream motion cone. All motion prediction methods, except the bounded conic motion prediction, are positively inclusive. }
\label{fig.UnicycleFeedbackMotionPredictionPositiveInclusion}
\vspace{-3mm}
\end{figure*}

\subsection{Unicycle Feedback Motion Prediction}
\label{sec.UnicycleFeedbackMotionPrediction}

Feedback motion prediction, which plays a key role in robot safety assessment and safe robot motion design, aims at determining a motion range bound (e.g., a positively invariant Lyapunov level set) on the closed-loop motion trajectory of a robotic system starting from a known initial state towards a given goal state under a specific control policy \cite{isleyen_arslan_RAL2022, arslan_isleyen_arXiv2023}. 
We now present several feedback motion range prediction methods that can be used for bounding the closed-loop motion trajectory of the kinematic unicycle robot model in \refeq{eq.UnicycleDynamics} under the unicycle forward motion control\footnote{In fact, the proposed unicycle feedback motion prediction methods hold for any (forward) unicycle control approach that decreases the Euclidean distance and perpendicular distance to the goal as well as has the finite-time and persistent goal alignment properties described in Lemmas \ref{lem.EuclideanDistance2Goal}-\ref{lem.PerperdicularGoalAlignmentDistance}. For example, one can alternatively use the proposed unicycle feedback motion prediction algorithms for the (forward) unicycle control policy in \cite{lee_etal_IROS2000}.} in \refeq{eq.forward_unicycle_control}.

\subsubsection{Circular Motion Range Prediction}
\label{sec.CircularMotionPrediction}

A classical approach for feedback motion prediction design is the use of invariant Lyapunov level sets for feedback control systems with known Lyapunov functions \cite{blanchini_Automatica1999}.
Since the Euclidean distance of the unicycle position to the goal position is a valid Lyapunov function for the unicycle forward motion control (\reflem{lem.EuclideanDistance2Goal}), the closed-loop motion trajectory of the unicycle robot can be bounded in terms of Euclidean balls \cite{arslan_koditschek_IJRR2019}.

\begin{proposition} \label{prop.CircularUnicycleMotionPrediction}
\emph{(Circular Unicycle Motion Prediction)} Starting at $t = 0$ from any initial pose $(\pos_0, \ort_0) \in \R^2 \times [-\pi, \pi)$, the unicycle robot position trajectory $\pos(t)$ under the unicycle forward motion control $\ctrl_{\goal}$ in \refeq{eq.forward_unicycle_control} towards a given goal $\goal \in \R^{2}$ is contained for all future times in the circular motion prediction set $\motionset_{\ctrl_{\goal}, \ball}(\pos_0, \ort_0)$ that is defined~as 
\begin{align}
\pos(t) \in \motionset_{\ctrl_{\goal}, \ball}(\pos_0, \ort_0) :=  \ball\plist{\goal, \norm{\goal - \pos_0}} \quad  \quad \forall t \geq 0,
\end{align}
where $\ball(\ctr,\radius) := \clist{ \vect{z} \in \R^{2} \big|  \norm{\vect{z} - \ctr } \leq \radius} $ is the Euclidean closed ball centered at $\vect{c} \in \R^{2}$ with radius $\radius \geq 0$.    
\end{proposition}
\begin{proof}
The result directly follows from \reflem{lem.EuclideanDistance2Goal}.
\end{proof} 

In addition to being positively invariant \cite{khalil_NonlinearSystems2001}, the circular unicycle motion prediction is positively inclusive, see \reffig{fig.UnicycleFeedbackMotionPredictionPositiveInclusion}. 

\begin{proposition} \label{prop.PositiveInclusionCircularMotionPrediction}
(\emph{Positive Inclusion of Circular Motion Prediction})
Under the unicycle forward motion control $\ctrl_{\goal}$ towards a goal position $\goal \in \R^2$, the circular motion prediction set $\motionset_{\ctrl_{\goal}, \ball}(\pos,\ort)$ is positively inclusive along the unicycle motion trajectory $(\pos(t), \ort(t))$, i.e.,
\begin{align}
\motionset_{\ctrl_{\goal}, \ball}(\pos(t), \ort(t)) \supseteq \motionset_{\ctrl_{\goal}, \ball}(\pos(t'), \ort(t'))\quad \forall t' \geq t. 
\end{align} 
\end{proposition}
\begin{proof}
See \refapp{app.PositiveInclusionCircularMotionPrediction}.
\end{proof}

It is important to remark that the positive inclusion of feedback motion prediction ensures that the safety assessment of the predicted robot motion is consistent for all future times since the predicted motion range shrinks over time.

\subsubsection{Conic Motion Range Prediction}
Although the circular unicycle motion prediction $\motionset_{\ctrl_{\goal}, \mathrm{B}}(\pos, \ort)$ has a simple form and comes with the positive invariance/inclusion property, it only depends on the Euclidean distance of the unicycle position $\pos$ to the goal position $\goal$ and it is independent of the unicycle orientation $\ort$.
As an alternative approach, in order to capture unicycle motion direction better, we introduce a new conic unicycle motion prediction that bounds the closed-loop motion trajectory of the unicycle robot model under the forward motion control based on the goal alignment error.

\begin{proposition} \label{prop.UnboundedConicUnicycleMotionPrediction}
\emph{(Unbounded Conic Motion Prediction)}
Starting at $t=0$ from any initial pose $(\pos_0, \ort_0) \in \R^{2} \times [-\pi, \pi)$, the unicycle position trajectory $\pos(t)$ under the forward motion control $\ctrl_{\goal}$ in \refeq{eq.forward_unicycle_control} is contained in the \emph{unbounded conic motion prediction} set $\motionset_{\ctrl_{\goal}, \mathrm{UC}}(\pos_0, \ort_0)$, i.e.,  $\pos(t) \in \motionset_{\ctrl_{\goal}, \mathrm{UC}}(\pos_0, \ort_0)$ for all $t \geq 0$, that is defined as
\begin{align}\label{eq.UnboundedConicUnicycleMotionPrediction}
\!\!\motionset_{\ctrl_{\goal}, \mathrm{UC}}(\pos, \ort)\!:=\! \!\left \{ 
\begin{array}{@{\!}l@{}l@{}}
 \cone(\pos, \goal, \dist_{y}(\pos, \ort)\!) & \text{, if } \ovecTsmall{\ort}\!\!\!(\goal \!-\! \pos) \! \geq \! 0   \\
 \hplane(\pos,\goal) & \text{, otherwise}
 \end{array}
 \right. \!\!\!
\end{align}
where $\cone(\apex, \base, \radius):=\clist{\apex + \alpha (\vect{z} - \apex) \Big| \alpha \geq 0, \vect{z} \in \ball(\base, \radius)}$ denotes the cone with apex point $\apex \in \R^2$ , base point $\base\in \R^2$, and  base-distance-to-cone-boundary $\radius \geq 0$, and   $\dist_{y}(\pos, \ort)$ is the perpendicular goal alignment distance defined as in \refeq{eq.PerpendicularAlignmentDistance}, and $\hplane(\apex,\base):=\clist{\vect{z} \in \R^{2} \big| \tr{(\base-\apex)}(\vect{z} - \apex)\geq 0}$ is the half-plane bounded at $\apex \in \R^2$ pointing towards $\base \in \R^2$.   
\end{proposition}
\begin{proof}
See \refapp{app.UnboundedConicUnicycleMotionPrediction}.
\end{proof}

\noindent Note that $\cone(\pos, \goal, \dist_{y}(\pos, \ort)) \subseteq \hplane(\pos,\goal)$ for any unicycle pose $(\pos,\ort) \in \R^2 \times [-\pi, \pi)$, and an interesting special case is that $\cone(\pos, \goal, \dist_{y}(\pos, \ort)) = \hplane(\pos,\goal)$ when $\ovecTsmall{\ort}\!(\goal \!-\! \pos) = 0$. 
Hence, the unbounded motion cone $\motionset_{\ctrl_{\goal}, \mathrm{UC}}(\pos, \ort)$ is continuous.

Although the unbounded conic motion prediction $\motionset_{\ctrl_{\goal}, \mathrm{UC}}(\pos, \ort)$ represents the unicycle motion direction more accurately compared to the circular motion prediction $\motionset_{\ctrl_{\goal}, \mathrm{\ball}}(\pos, \ort)$, the potentially occupied region by the robot motion is predicted more conservatively due to its unboundedness (like \emph{velocity obstacles} \cite{fiorini_shiller_IJRR1998}).  
To combine the nice feature of these motion predictions, we simply take the intersection of circular and conic motion predictions and define the \emph{bounded conic unicycle motion prediction} set $\motionset_{\ctrl_{\goal}, \mathrm{BC}}(\pos, \ort)$  for any unicycle pose $(\pos, \ort) \!\in\! \R^2 \times [-\pi, \pi)$ as
\begin{align}
&\motionset_{\ctrl_{\goal}, \mathrm{BC}}(\pos, \ort) := \motionset_{\ctrl_{\goal}, \mathrm{B}}(\pos, \ort) \cap \motionset_{\ctrl_{\goal}, \mathrm{UC}}(\pos, \ort)
\\
& =\left \{ 
\begin{array}{@{}l@{}l@{}}
 \ball(\goal, \norm{\goal - \pos}) \bigcap \cone(\pos, \goal, \dist_{y}(\pos, \ort)) & \text{, if } \ovecTsmall{\ort}\!\!(\goal \!-\! \pos) \geq 0   \\
 \ball(\goal, \norm{\goal - \pos}) & \text{, otherwise}
 \end{array}
 \right. \nonumber
\end{align}
which, by construction, is a valid feedback motion prediction for the unicycle forward motion control \mbox{(Propositions \ref{prop.CircularUnicycleMotionPrediction}\&\ref{prop.UnboundedConicUnicycleMotionPrediction})}.
The intersection of circular and conic motion predictions results in an accurate and bounded motion prediction, but the bounded conic motion prediction $\motionset_{\ctrl_{\goal}, \mathrm{BC}}(\pos, \ort)$ does not inherit the positive inclusion property (see \reffig{fig.UnicycleFeedbackMotionPredictionPositiveInclusion}) from the circular motion prediction (\refprop{prop.PositiveInclusionCircularMotionPrediction}), since the unbounded conic motion prediction is positively variant, which can be resolved by properly bounding the conic motion prediction as described below.

\subsubsection{Ice-Cream-Cone-Shaped Motion Range Prediction}

As opposed to their intersection, an elegant way of combining circular and conic motion predictions is by bounding the conic motion prediction $\motionset_{\ctrl_{\goal}, \mathrm{UC}}(\pos, \ort)$ with the largest circular motion prediction $\ball(\goal, \dist_{\goal}(\pos, \ort))$ contained in the cone $\motionset_{\ctrl_{\goal}, \mathrm{UC}}(\pos, \ort)$, which yields a more accurate feedback motion prediction with a positive inclusion property, see \reffig{fig.UnicycleFeedbackMotionPredictionPositiveInclusion}.

\begin{proposition} \label{prop.IceCreamUnicycleMotionPrediction}
\emph{(Ice-Cream-Cone-Shaped Unicycle Motion Prediction)}
For any goal position $\goal \in \R^{2}$ and any initial unicycle pose $(\pos_0, \ort_0) \in \R^{2} \times [-\pi, \pi)$ at $t=0$, 
the unicycle position trajectory $\pos(t)$ under the forward motion control $\ctrl_{\goal}$ in \refeq{eq.forward_unicycle_control} is contained for all future times in the \emph{ice-cream-cone-shaped motion prediction} set $\motionset_{\ctrl_{\goal}, \mathrm{IC}}(\pos_0, \ort_0)$, i.e., $\pos(t) \in \motionset_{\ctrl_{\goal}, \mathrm{IC}}(\pos_0, \ort_0)$ for all $t \geq 0$, that is defined as
\begin{align}
\motionset_{\ctrl_{\goal}, \mathrm{IC}}(\pos, \ort):= 
\left\{ 
\begin{array}{@{}l@{}l@{}}
\icone(\pos, \goal, \dist_{\goal}(\pos, \ort)\!)  & \text{, if } \ovecTsmall{\ort}\!\!(\goal \!-\! \pos) \geq 0   \\
\ball(\goal, \norm{\goal - \pos}) & \text{, otherwise}
\end{array}
\right. \nonumber
\end{align}
where  $\dist_{y}(\pos, \ort)$ is the perpendicular alignment distance in \refeq{eq.PerpendicularAlignmentDistance} and the bounded ice-cream cone $\icone(\apex, \base, \radius)$ is defined as
\begin{align}
\icone\plist{\apex, \base, \radius} &:= \clist{\apex + \alpha(\vect{z} - \apex) \Big | \alpha \in [0,1], \vect{z} \in \ball(\base, \radius)} \\
&= \conv\plist{\apex, \ball(\base, \radius)}.
\end{align}
 Here, $\conv$ denotes the convex hull operator.
\end{proposition}
\begin{proof}
See \refapp{app.IceCreamUnicycleMotionPrediction}
\end{proof}

\noindent Note that $\icone(\pos, \ort, \dist_{\goal}(\pos, \ort)) \subseteq \ball(\goal, \norm{\goal \! -\! \pos})$ for any unicycle pose $(\pos, \ort) \in \R^2 \times [-\pi, \pi)$ where the equality holds for \mbox{$\ovecTsmall{\ort}\!\!(\goal\! -\! \pos) = 0$}. 
Hence, the ice-cream-cone-shaped motion prediction $\motionset_{\ctrl_{\goal}, \mathrm{IC}}(\pos, \ort)$ is continuous.

\begin{proposition}\label{prop.PositiveInclusionIceCreamMotionPrediction}
\emph{(Positive Inclusion of Ice-Cream Motion Cone)}
For any goal position $\goal \in \R^2$ and any initial unicycle pose $(\pos_0, \ort_0) \in \R^{2} \times [-\pi, \pi)$, the ice-cream-cone-shaped motion prediction $\motionset_{\ctrl_{\goal}, \mathrm{IC}}(\pos, \ort)$ is positively inclusive along the unicycle motion trajectory $(\pos(t), \ort(t))$ of the forward unicycle motion control $\ctrl_{\goal}$ in \refeq{eq.forward_unicycle_control}, i.e,
\begin{align}
\motionset_{\ctrl_{\goal}, \mathrm{IC}}(\pos(t), \ort(t)) \supseteq \motionset_{\ctrl_{\goal}, \mathrm{IC}}(\pos(t'), \ort(t')) \quad \forall t'\geq t.
\end{align}  
\end{proposition}
\begin{proof}
See \refapp{app.PositiveInclusionIceCreamMotionPrediction}.
\end{proof}

Finally, the decreasing perpendicular goal alignment distance (\reflem{lem.PerperdicularGoalAlignmentDistance}) implies that the signed goal alignment distance $\sdist_{\goal}(\pos, \ort):= \nvecTsmall{\ort}(\goal - \pos)$ has the same sign under the forward unicycle motion control. 
Accordingly, the ice-cream motion cone can be truncated (in half) to obtain a tighter motion bound with the cost of losing convexity.
\begin{proposition}\label{prop.TruncatedIceCreamMotionCone}
\emph{(Truncated Ice-Cream Motion Cone)}
For any goal position  $\goal \in \R^2$ and any initial pose  $(\pos_0, \ort_0) \in \R^{2} \times [-\pi, \pi) $, the forward unicycle motion control $\ctrl_{\goal}$ in \refeq{eq.forward_unicycle_control} keeps the unicycle position trajectory $\pos(t)$ for all future times $t \geq 0$ inside the \emph{truncated ice-cream-cone-shared motion prediction} set $\motionset_{\ctrl_{\goal}, \mathrm{TC}}(\pos_0, \ort_0)$, i.e., $\pos(t) \in \motionset_{\ctrl_{\goal}, \mathrm{TC}}(\pos_0, \ort_0)$ for all $t \geq 0$,  that is defined~as
\begin{align}
\motionset_{\ctrl_{\goal}, \mathrm{TC}}(\pos, \ort) =  \left\{ 
\begin{array}{@{}l@{}l@{}}
\tcone(\pos, \goal, \ort )  & \text{, if } \ovecTsmall{\ort}\!\!(\goal \!-\! \pos) \geq 0   \\
\ball(\goal, \norm{\goal - \pos}) & \text{, otherwise}
\end{array}
\right. \nonumber
\end{align}
where the truncated ice-cream cone $\tcone(\apex, \base, \ort)$ associated with apex point $\apex \in \R^2$, base point $\base \in \R^2 $ and boundary orientation angle $\ort \in [-\pi, \pi)$ is defined in terms of the goal alignment distance $\dist_{\goal}(\pos, \ort)$ in \refeq{eq.PerpendicularAlignmentDistance} as
{\small
\begin{align}
\tcone(\apex, \base, \ort) 
&:=\conv\!\plist{\!\apex, \base, \apex \!+\! \ovectsmall{\ort} \! \ovecTsmall{\ort} \!\!(\base \! - \!\apex)\!\!} \nonumber 
\!\cup  \ball(\base, \dist_{\base}(\apex, \ort)\!).
\end{align}
}%
\end{proposition}
\begin{proof}
See \refapp{app.TruncatedIceCreamMotionCone}.
\end{proof}

\vspace{-3mm}

\begin{proposition}\label{prop.PositiveInclusionTruncatedIceCreamMotionCone}
\emph{(Positive Inclusion of Truncated Ice-Cream Motion Cone)}
The truncated ice-cream motion cone is positively inclusive along the unicycle motion trajectory $(\pos(t), \ort(t))$ of the forward motion control $\ctrl_{\goal}$ in \refeq{eq.forward_unicycle_control}, i.e.,
\begin{align}
\motionset_{\ctrl_{\goal}, \mathrm{TC}}(\pos(t), \ort(t)) \supseteq \motionset_{\ctrl_{\goal}, \mathrm{TC}}(\pos(t'), \ort(t')) \quad \forall t' \geq t. 
\end{align}
\end{proposition}
\begin{proof}
See \refapp{app.PositiveInclusionTruncatedIceCreamMotionCone}.
\end{proof}

As a final remark, we find it useful to highlight the inclusion relation of the bounded unicycle feedback motion prediction methods, as illustrated in \reffig{fig.motion_prediction_demo}  
\begin{proposition}\label{prop.InclusionRelation}
For any goal position $\goal \in \R^2$ and any unicycle pose $(\pos, \ort) \in \R^2 \times [-\pi, \pi)$, the aforementioned feedback motion prediction methods for the unicycle forward motion control $\ctrl_{\goal}$ in \refeq{eq.forward_unicycle_control} satisfy
\begin{align}
\motionset_{\ctrl_{\goal}, \mathrm{TC}}(\pos, \ort) \!\subseteq\!  \motionset_{\ctrl_{\goal}, \mathrm{IC}}(\pos, \ort)\! \subseteq \! \motionset_{\ctrl_{\goal}, \mathrm{BC}}(\pos, \ort) \! \subseteq \! \motionset_{\ctrl_{\goal}, \mathrm{B}}(\pos, \ort). \nonumber
\end{align}
\end{proposition}
\begin{proof}
See \refapp{app.InclusionRelation}.
\end{proof}

\section{Safe Unicycle Robot Navigation} 
\label{sec.SafeNavigation}

In this section, we demonstrate an application of unicycle feedback motion prediction for safe robot navigation using a reference governor \cite{isleyen_arslan_RAL2022}.
In brief, the governed feedback motion design framework \cite{isleyen_arslan_RAL2022} allows for extending the applicability of a reference motion planner that is designed for the fully actuated kinematic robot model to the nonholonomically constrained kinematic unicycle robot model using unicycle feedback motion prediction and safety assessment, as illustrated in \reffig{fig.general_framework}.

For ease of exposition, we consider a disk-shaped unicycle robot of body radius $\radius>0$, centered at position $\pos \in \workspace$ with orientation $\ort \in [ -\pi, \pi )$, that operates in a known static closed compact environment $\workspace \subseteq \R^{2}$ which is cluttered with a collection of obstacles represented by an open set $\obstspace \subset \R^{2}$. 
Hence, the robot's free space, denoted by $\freespace$, of collision-free unicycle positions is given by
\begin{align} \label{eq.FreeSpace}
\freespace \ldf \clist{ \pos \in \workspace \, \big| \,   \ball(\pos,\radius) \subseteq \workspace \setminus \obstspace }.
\end{align}
To ensure global navigation between any start and goal positions in $\freespace$, we assume the free space $\freespace$ is path-connected.

Moreover, suppose $\refplan_{\globalgoal}:\refdomain \rightarrow \R^2$ is a Lipschitz continuous \emph{reference motion planner} for the first-order fully actuated robot dynamics (i.e., $\dot{\pos} = \refplan_{\globalgoal}(\pos)$) that asymptotically brings all robot positions in its positively-invariant collision-free domain $\refdomain \subseteq \freespace$ to a desired global goal position $\globalgoal \in \refdomain$ while avoiding collisions along the way \cite{arslan_koditschek_ICRA2017}.   
For the kinematic fully-actuated robot model, one can construct such a reference vector field using off-the-shelf motion
planning algorithms \cite{lavalle_PlanningAlgorithms2006, choset_etal_PrinciplesOfRobotMotion2005}; 
for example, we use the path pursuit vector field planner \cite{arslan_koditschek_ICRA2017} in our numerical simulations in \refsec{sec.NumericalSimulations}. 
We below describe how to safely follow the reference vector field planner $\refplan_{\globalgoal}$ by a kinematic unicycle robot using a reference governor and the safety assessment of the predicted unicycle feedback motion.  

\begin{figure}[t]   
\centering
\includegraphics[width = \columnwidth]{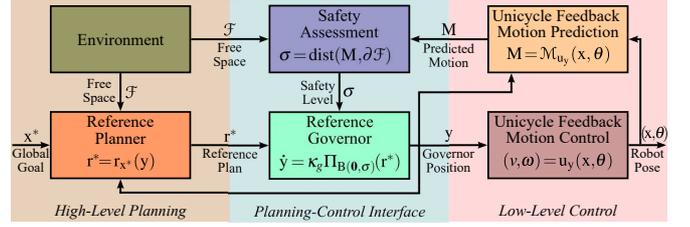}
\caption{Safe unicycle navigation framework using feedback motion prediction and reference governor that establishes a bidirectional safety interface between high-level motion planning and low-level motion control based on the predicted unicycle motion relative to the governor position.}
\label{fig.general_framework}
\end{figure}

\subsection{Unicycle Motion Safety Assessment}
\label{sec.SafetyAssessment}

The availability of a unicycle feedback motion prediction $\motionset_{\ctrl_{\goal}}(\pos, \ort)$ associated with the forward motion control $\ctrl_{\goal}$ allows one to effectively check the safety of the closed-loop unicycle motion trajectory $\pos(t)$, starting at $t = 0$ from a unicycle pose $(\pos_0, \ort_0) \in \workspace \times [-\pi, \pi)$, towards a give goal position $\goal \in \workspace$ in a cluttered environment $\workspace$, because having feedback motion prediction in the free space implies safe robot motion, i.e.,
\begin{align}
\motionset_{\ctrl_{\goal}}(\pos_0, \ort_0) \subseteq \freespace \Longrightarrow \pos(t) \in \freespace \quad \forall t \geq 0.
\end{align}
Accordingly, under the forward motion control $\ctrl_{\goal}$, we measure the safety level $\sigma$ of unicycle robot motion starting at a unicycle pose $(\pos, \ort) \in \R^2 \times [-\pi, \pi)$ towards a goal $\goal \in \R^2$ by the minimum distance between the predicted unicycle motion range $\motionset_{\ctrl_{\goal}}(\pos, \ort)$ and the free space boundary $\partial \freespace$ as 
\begin{subequations}
\begin{align}
\safelevel(\motionset_{\ctrl_{\goal}}(\pos, \ort)) &:= \safedist(\motionset_{\ctrl_{\goal}}(\pos, \ort), \partial \freespace)  
\\
& := \left \{ \begin{array}{@{}l@{\,}l@{}}
\min\limits_{\substack{\vect{a} \in \motionset_{\ctrl_{\goal}}(\pos, \ort)\\ \vect{b} \in \partial \freespace} } \norm{\vect{a} - \vect{b}} & \text{, if } \pos \in \freespace \\
0 & \text{, otherwise.}
\end{array}
\right.
\end{align}    
\end{subequations}
Here, a safety level of zero means unsafe motion; the higher the safety level the safer the motion. 
Note that we consider being exactly on the boundary of the free space to be unsafe although it is, by definition \refeq{eq.FreeSpace}, free of collisions.

A requirement of the safety measure for governed feedback motion design is that $\safelevel(\motionset_{\ctrl_{\goal}}(\pos, \ort)\!)$ is a locally Lipschitz continuous function of the unicycle pose $(\pos, \ort)$ and the goal position $\goal$ \cite{isleyen_arslan_RAL2022}.
\begin{proposition}\label{prop.LipschitzSafetyLevel}
For any unicycle motion prediction $\motionset_{\ctrl_{\goal}} \in \clist{\motionset_{\ctrl_{\goal}, \mathrm{B}}, \motionset_{\ctrl_{\goal}, \mathrm{BC}}, \motionset_{\ctrl_{\goal}, \mathrm{IC}}, \motionset_{\ctrl_{\goal}, \mathrm{TC}}}$,  the safety assessment $\safelevel(\motionset_{\ctrl_{\goal}}(\pos, \ort)) = \safedist(\motionset_{\ctrl_{\goal}}(\pos, \ort), \partial \freespace)$  is locally Lipschitz.  
\end{proposition}
\begin{proof}
See \refapp{app.LipschitzSafetyLevel}.
\end{proof}

\subsection{Unicycle-Governor Navigation Dynamics}

The safety assessment of predicted unicycle motion allows us to properly adapt a reference vector field planner $\refplan_{\globalgoal}$ for safe unicycle navigation via a reference governor.
A \textit{reference governor} is a first-order dynamical system (e.g., a virtual fully actuated robot) with position  $\govpos \in \R^{2}$ that follows the reference planner $\refplan_{\globalgoal}: \refdomain \rightarrow \R^{\dimspace}$ towards the goal $\globalgoal \in \refdomain \subseteq \freespace$ as close as possible, based on the safety level $\safelevel(\motionset_{\ctrl_{\govpos}}(\pos, \ort))$ of the predicted unicycle motion starting from $(\pos, \ort) \in \R^2 \times [-\pi, \pi)$ towards the governor position $\govpos$.

Accordingly, for any choice of bounded unicycle feedback motion prediction $\motionset_{\ctrl_{\govpos}}$ associated with the forward motion control $\ctrl_{\goal}(\pos, \ort) = (\linvel_{\goal}(\pos, \ort), \angvel_{\goal}(\pos, \ort))$ in \refeq{eq.forward_unicycle_control}, based on a standard form of the reference governor dynamics  \cite{arslan_koditschek_ICRA2017, isleyen_arslan_RAL2022},  we design the unicycle-governor navigation dynamics as:
\begin{subequations}\label{eq.UnicycleGovernorNavigationDynamics}
\begin{align} 
\dot{\pos} & = \linvel_{\goal}(\pos, \ort) \ovectsmall{\ort}, 
\\
\dot{\ort} & = \angvel_{\goal}(\pos, \ort), \\
\dot{\govpos} & = \govgain \proj_{\ball(\mat{0}, \safelevel(\motionset_{\ctrl_{\goal}}(\pos, \ort))}(\refplan_{\globalgoal}(\govpos)),  \label{eq.ReferenceGovernorDynamics} 
\end{align}  
\end{subequations}
where $\govgain > 0$ is a governor control gain, $\proj_{A}(\vect{b}) \ldf \argmin_{\vect{a} \in A} \norm{\vect{a} - \vect{b}}$ denotes the  metric projection of a point $\vect{b}$ onto a closed set $A$, and $\ball(\mat{0}, \sigma)$ is the Euclidean ball centered at the origin with radius $\sigma \geq 0$.
Our design ensures that the governor is only allowed to move according to the reference planner $\refplan_{\globalgoal}$ if the robot's motion relative to the governor is predicted to be safe, i.e., $\safelevel(\motionrange_{\ctrl_{\ctrlgoal}}(\pos, \ort)) = \safedist(\motionrange_{\ctrl_{\ctrlgoal}}(\pos, \ort), \partial \freespace) > 0$.
Also, note that the right-hand side of the reference governor dynamics in \refeq{eq.ReferenceGovernorDynamics} is Lipschitz continuous since both the safety level $\safelevel(\motionrange_{\ctrl_{\ctrlgoal}}(\pos, \ort))$ and the reference planner $\refplan_{\goal}(\govpos)$ are Lipschitz.

\begin{proposition}\label{prop.SafeStableUnicycleGovernorNavigation}
\emph{(Safe \& Stable Unicycle-Governor Navigation)}
Starting from any unicycle pose $(\pos, \ort) \in \R^2 \times [-\pi, \pi)$ and any governor position $\govpos \in \R^2$ with a strictly positive safety level $\safelevel(\motionrange_{\ctrl_{\goal}}(\pos, \ort)) > 0$, the unicycle-governor dynamics in \refeq{eq.UnicycleGovernorNavigationDynamics} asymptotically bring the unicycle robot and the governor to the global goal $\globalgoal$ according to the reference vector field $\refplan_{\globalgoal}$ with no collisions along the way.   
\end{proposition}
\begin{proof}
See \refapp{app.SafeStableUnicycleGovernorNavigation}.
\end{proof}

\begin{figure*}[t]
\centering
\begin{tabular}{@{}c@{\hspace{0.5mm}}c@{\hspace{0.5mm}}c@{\hspace{0.5mm}}c@{\hspace{0.5mm}}c@{}}
\includegraphics[width = 0.197\textwidth]{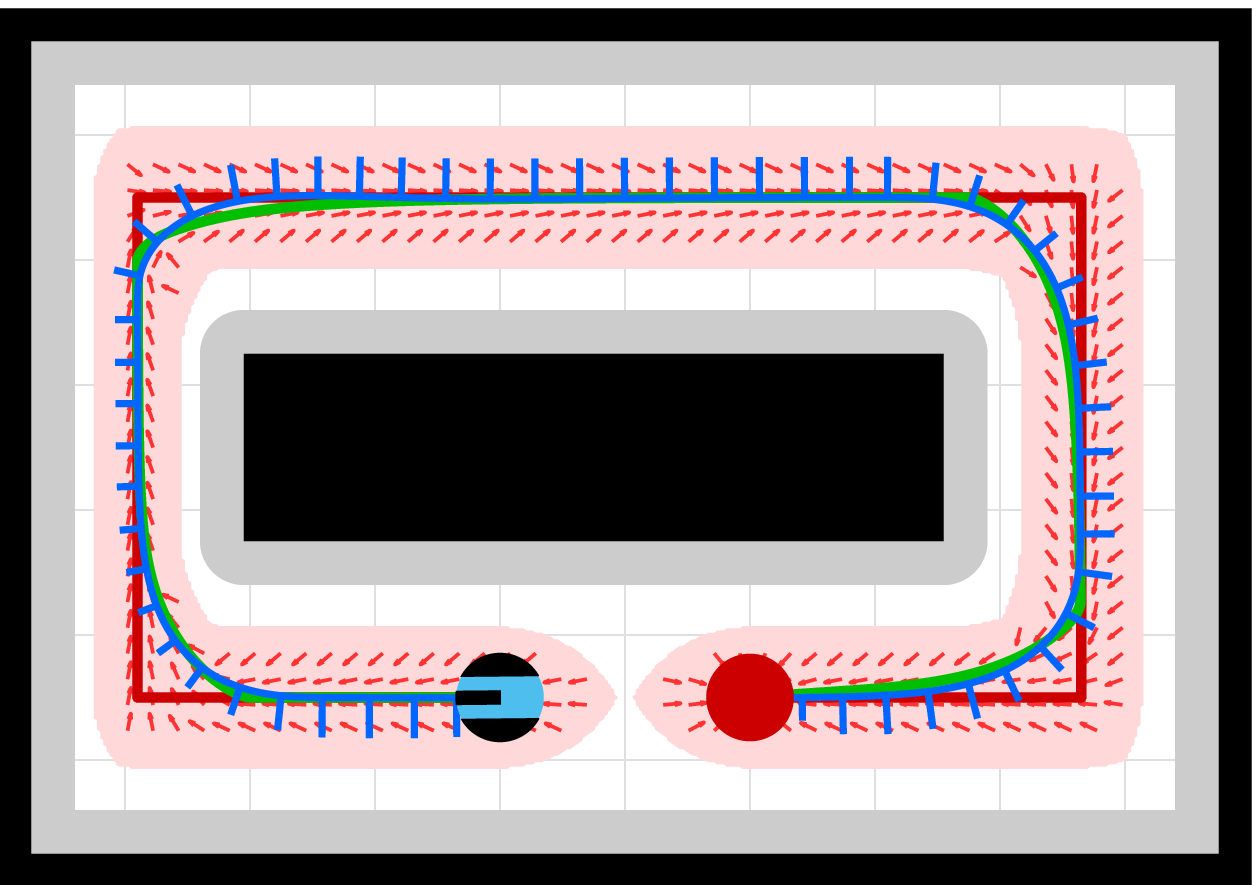} & 
\includegraphics[width = 0.197\textwidth]{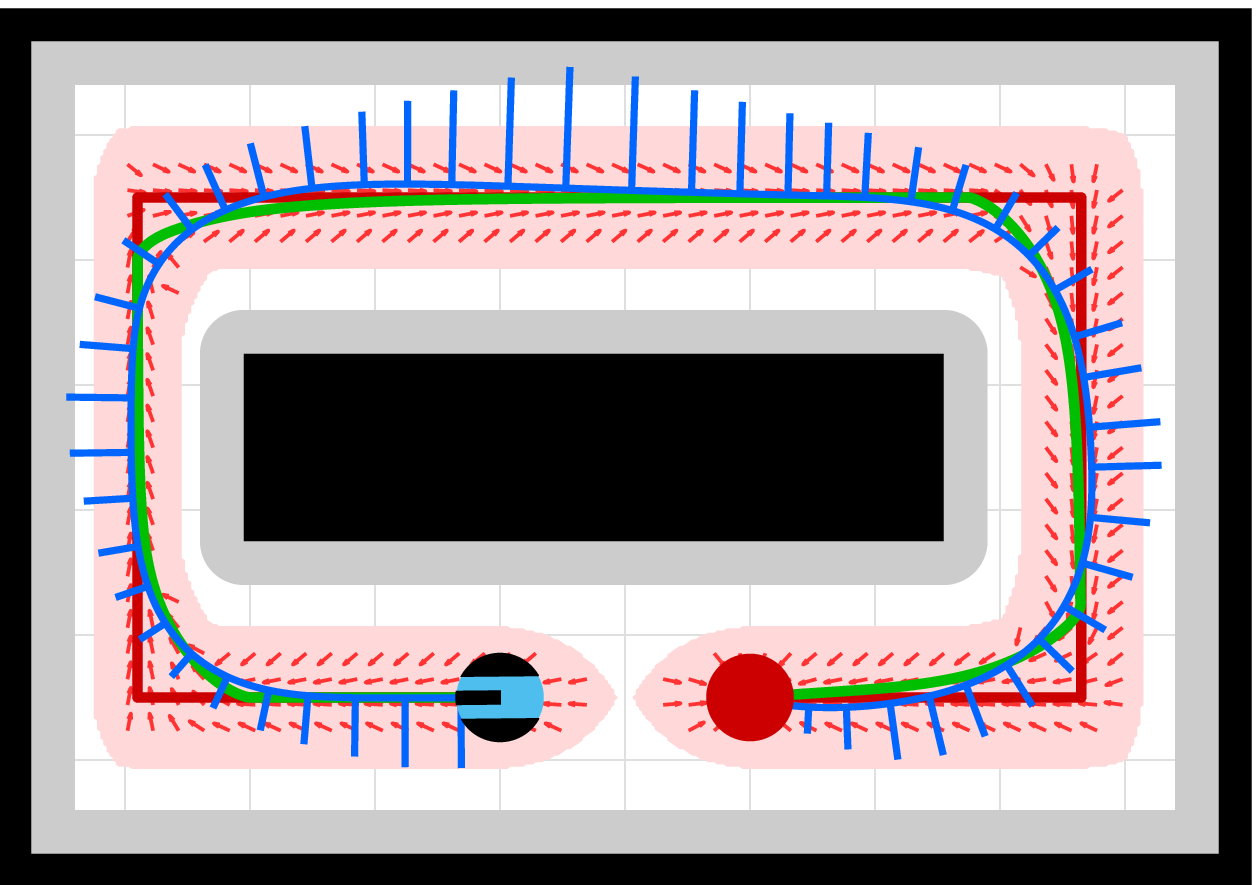} & 
\includegraphics[width = 0.197\textwidth]{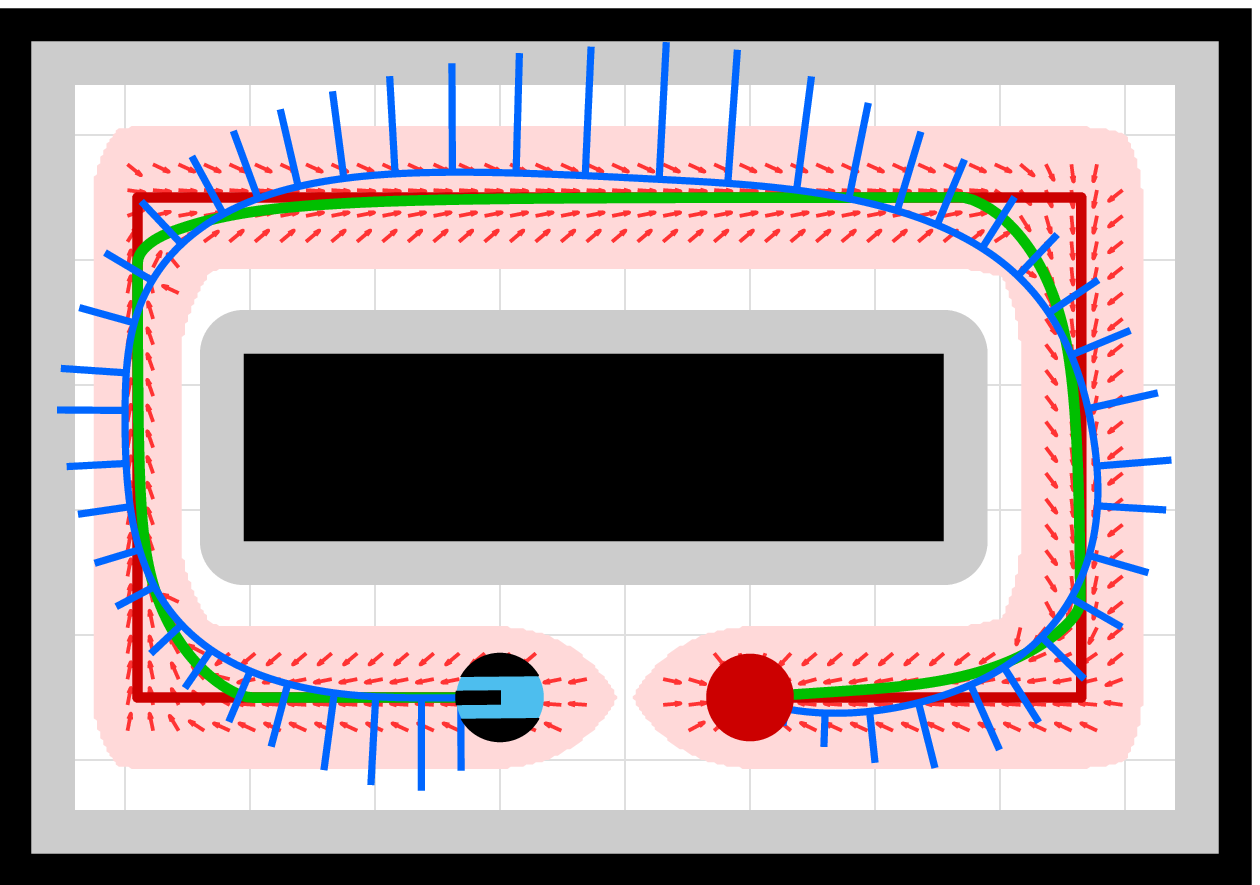} &  
\includegraphics[width = 0.197\textwidth]{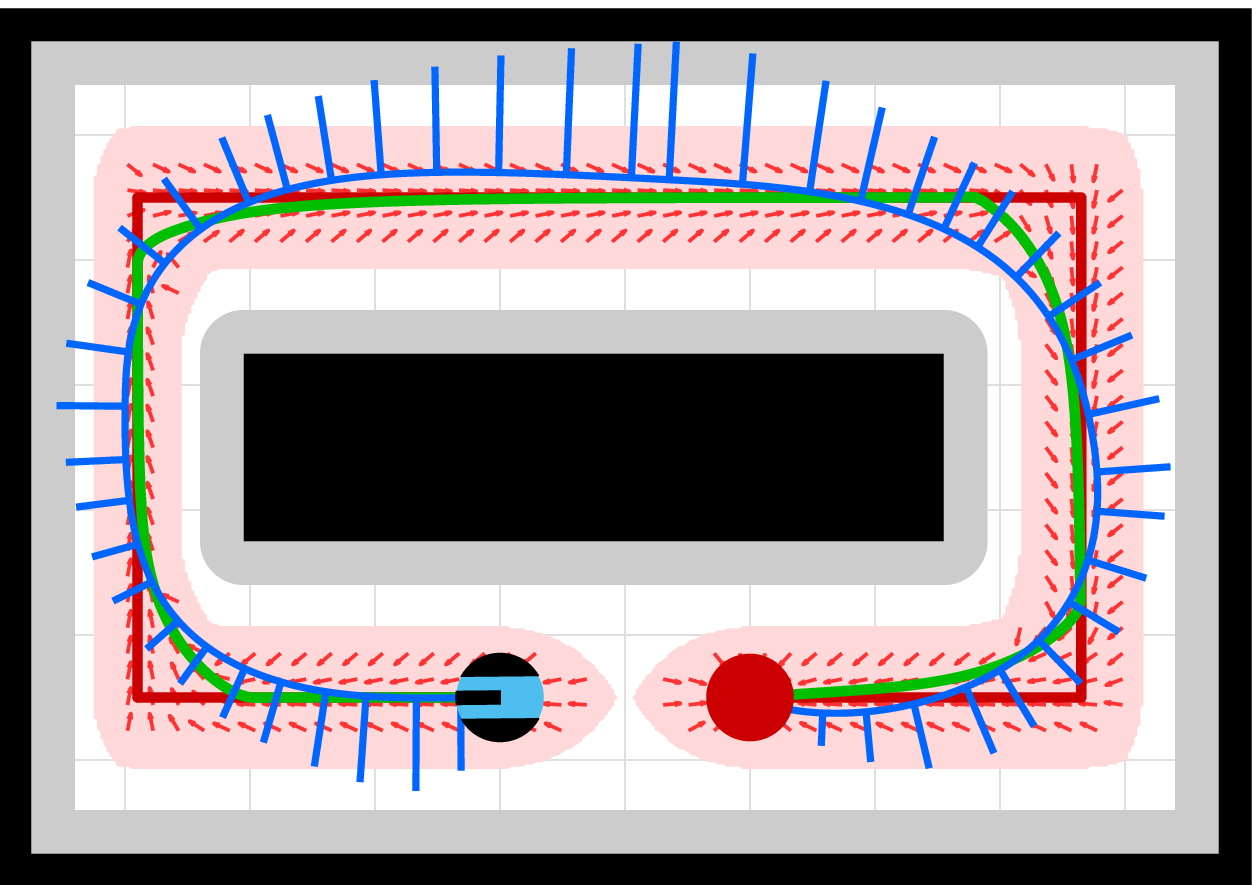} &  
\includegraphics[width = 0.197\textwidth]{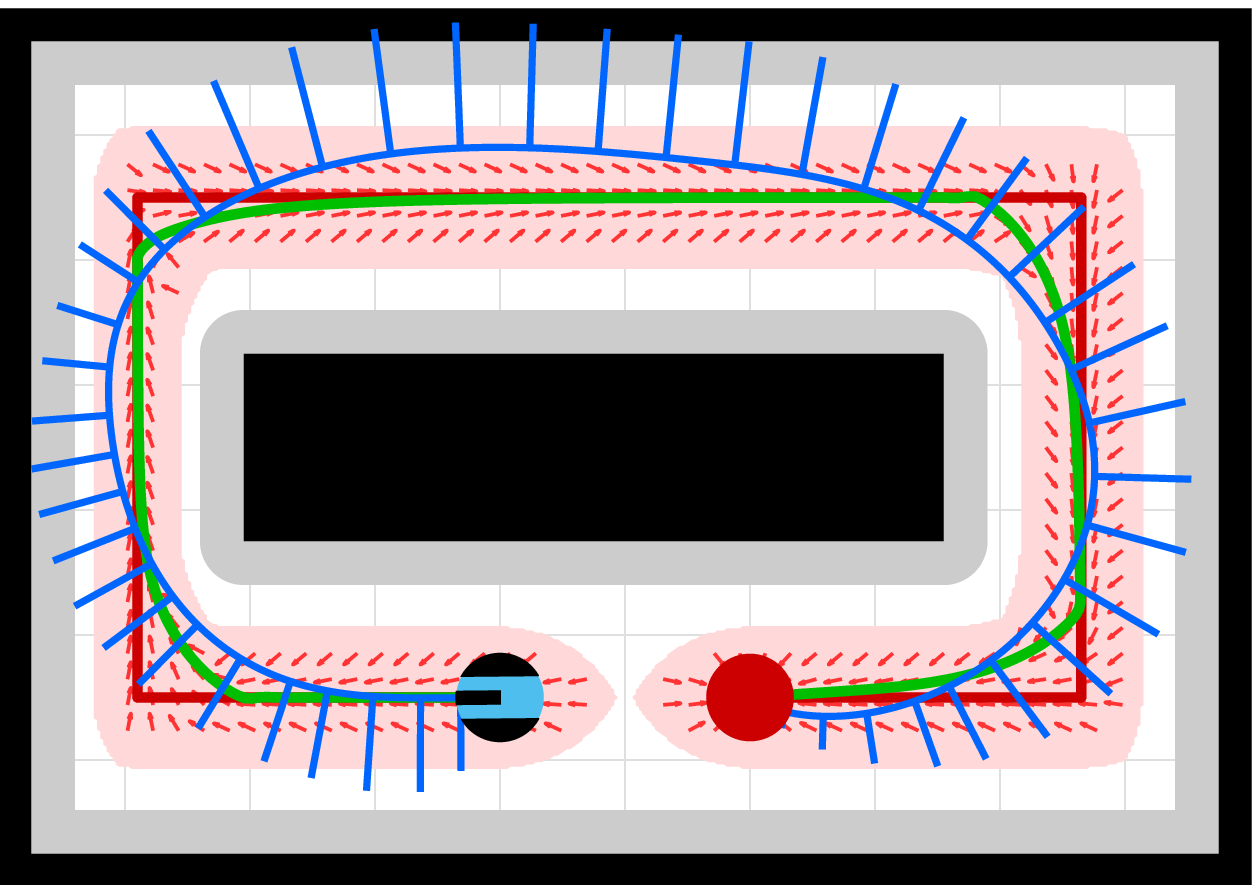} 
\\
\includegraphics[width = 0.197\textwidth]{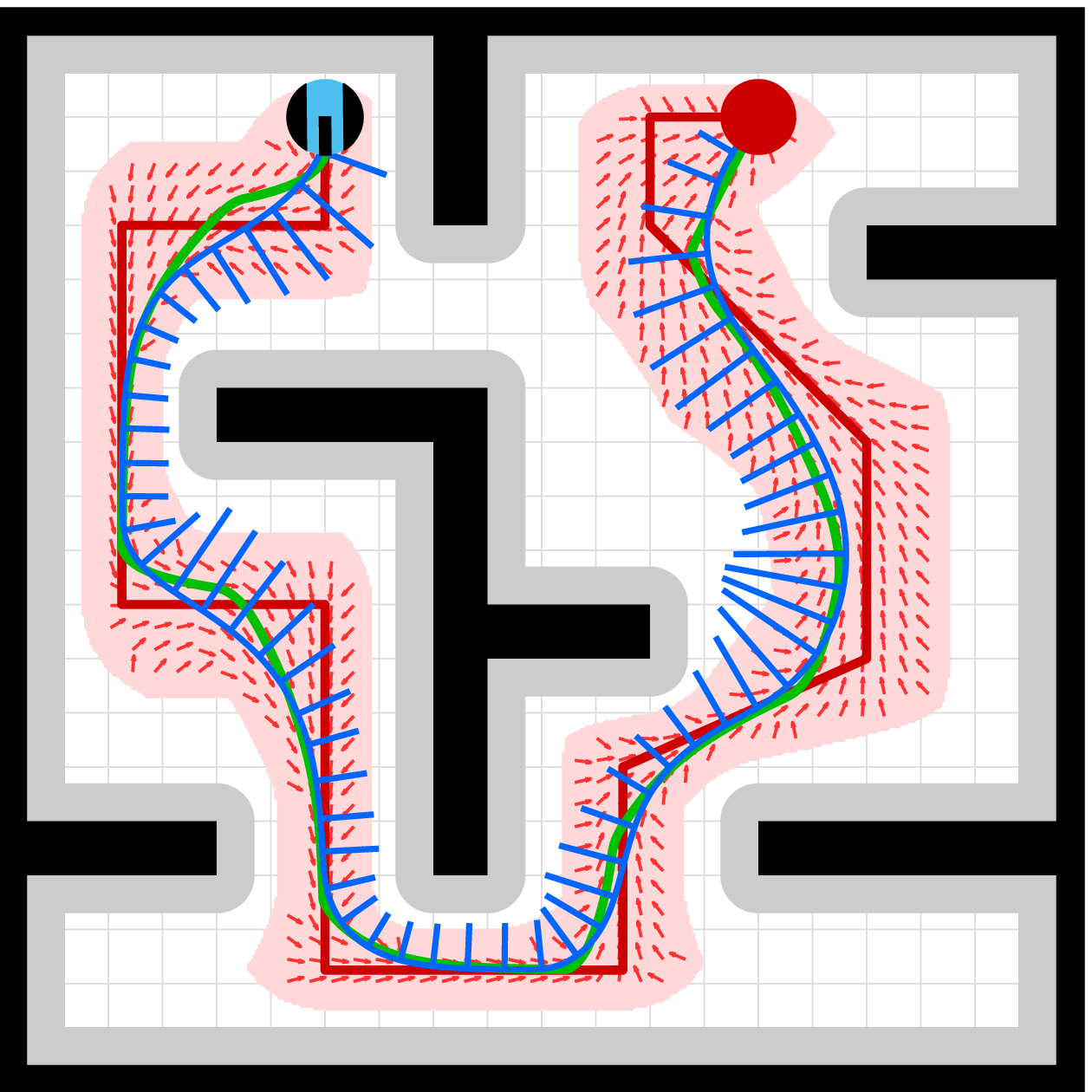} & 
\includegraphics[width = 0.197\textwidth]{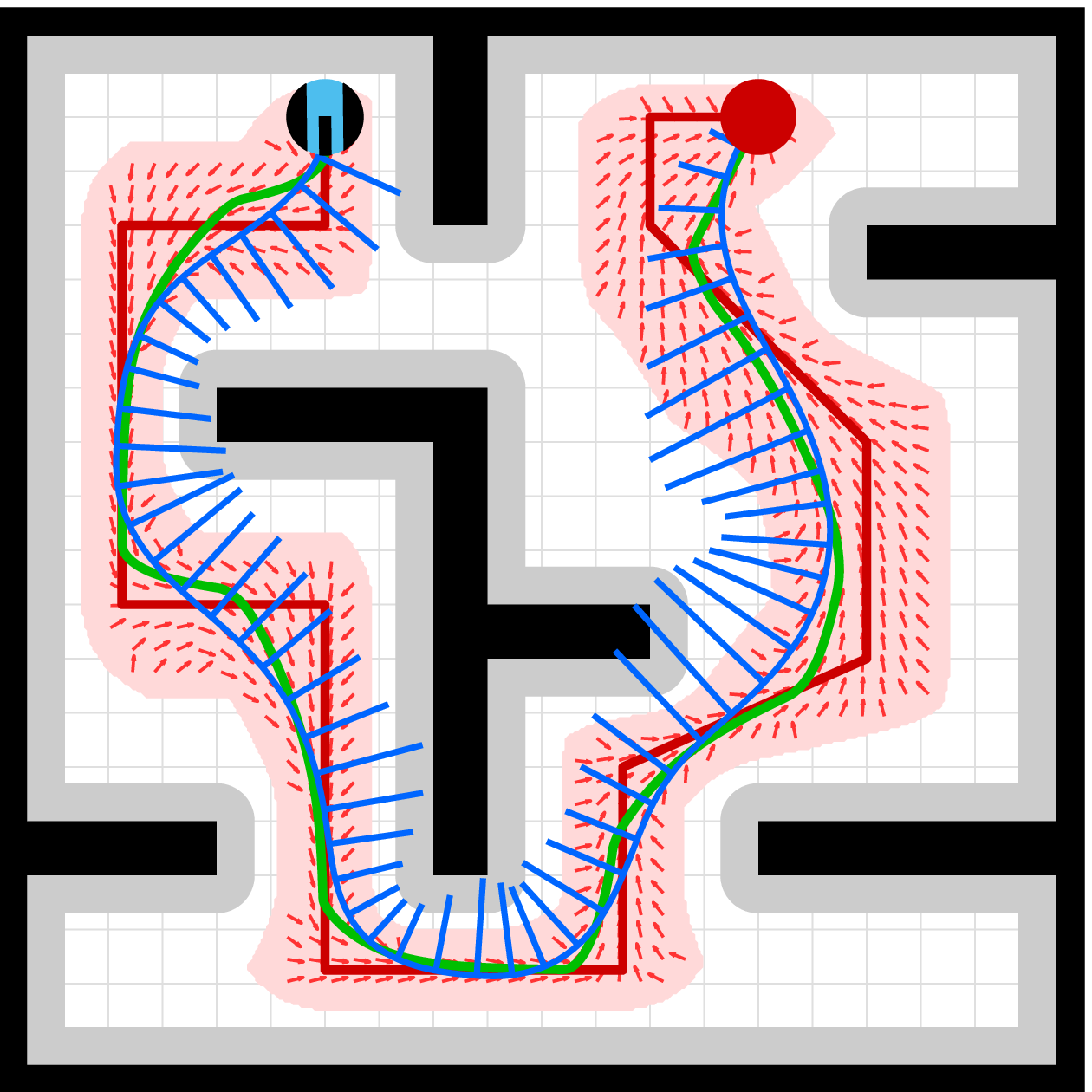} & 
\includegraphics[width = 0.197\textwidth]{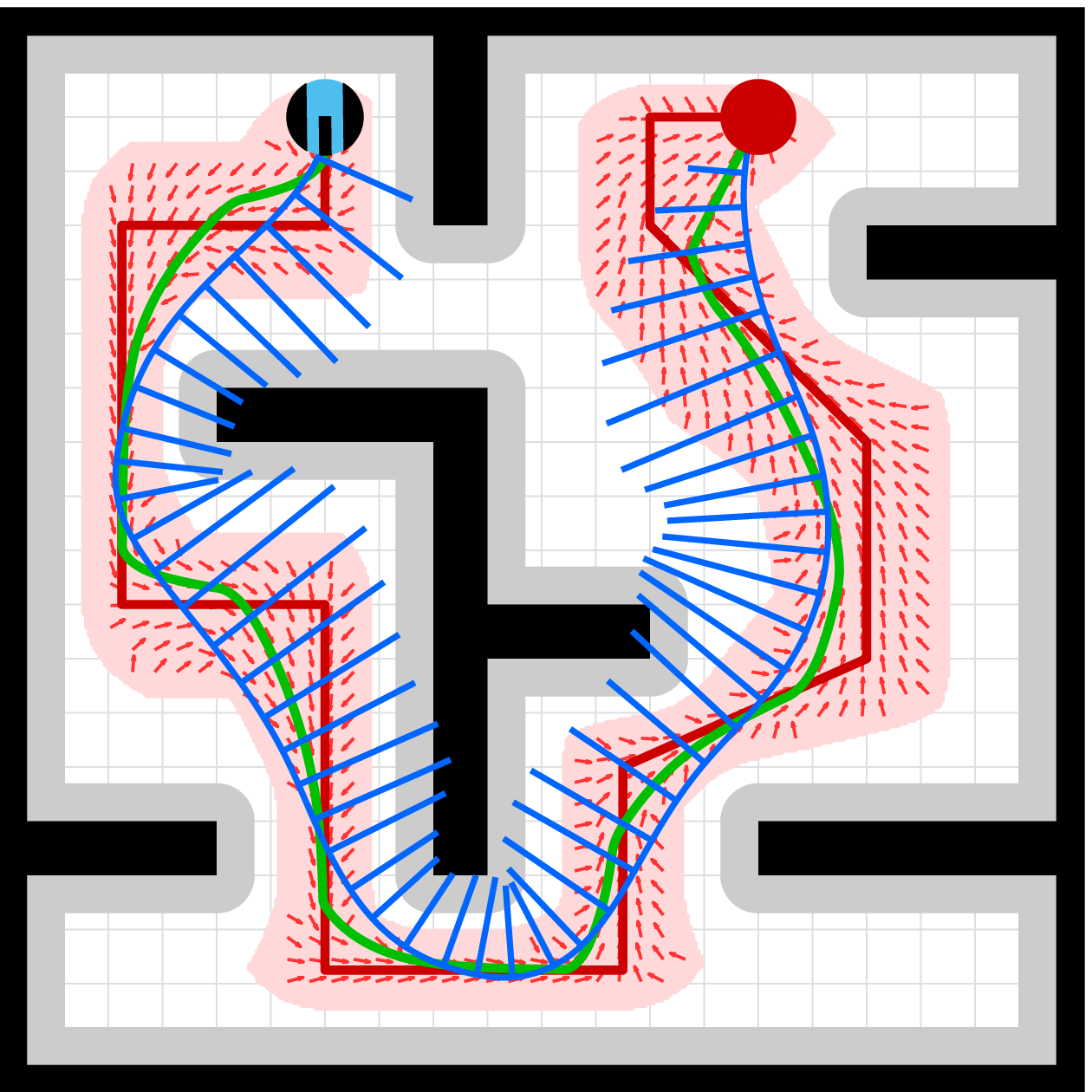} & 
\includegraphics[width = 0.197\textwidth]{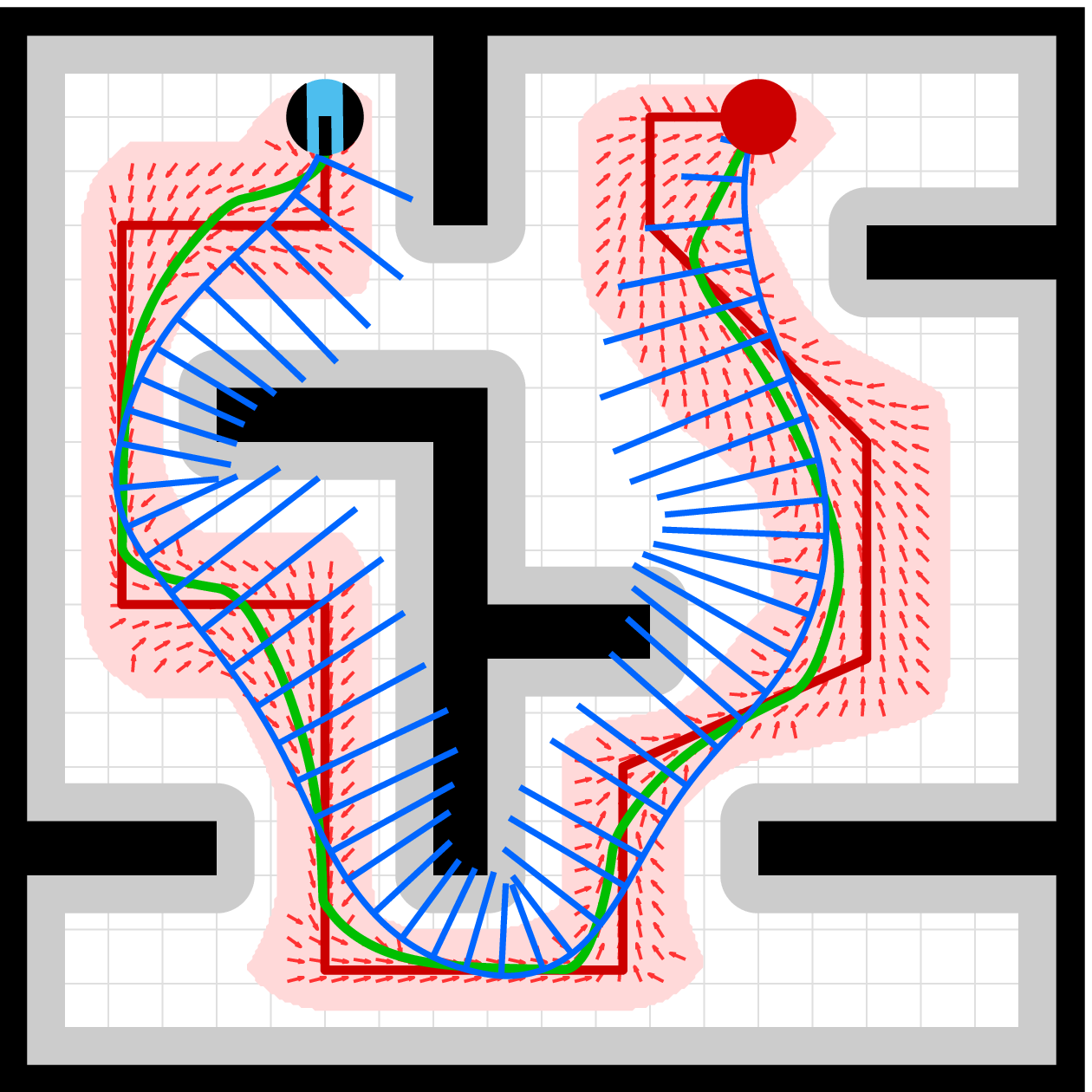} & 
\includegraphics[width = 0.197\textwidth]{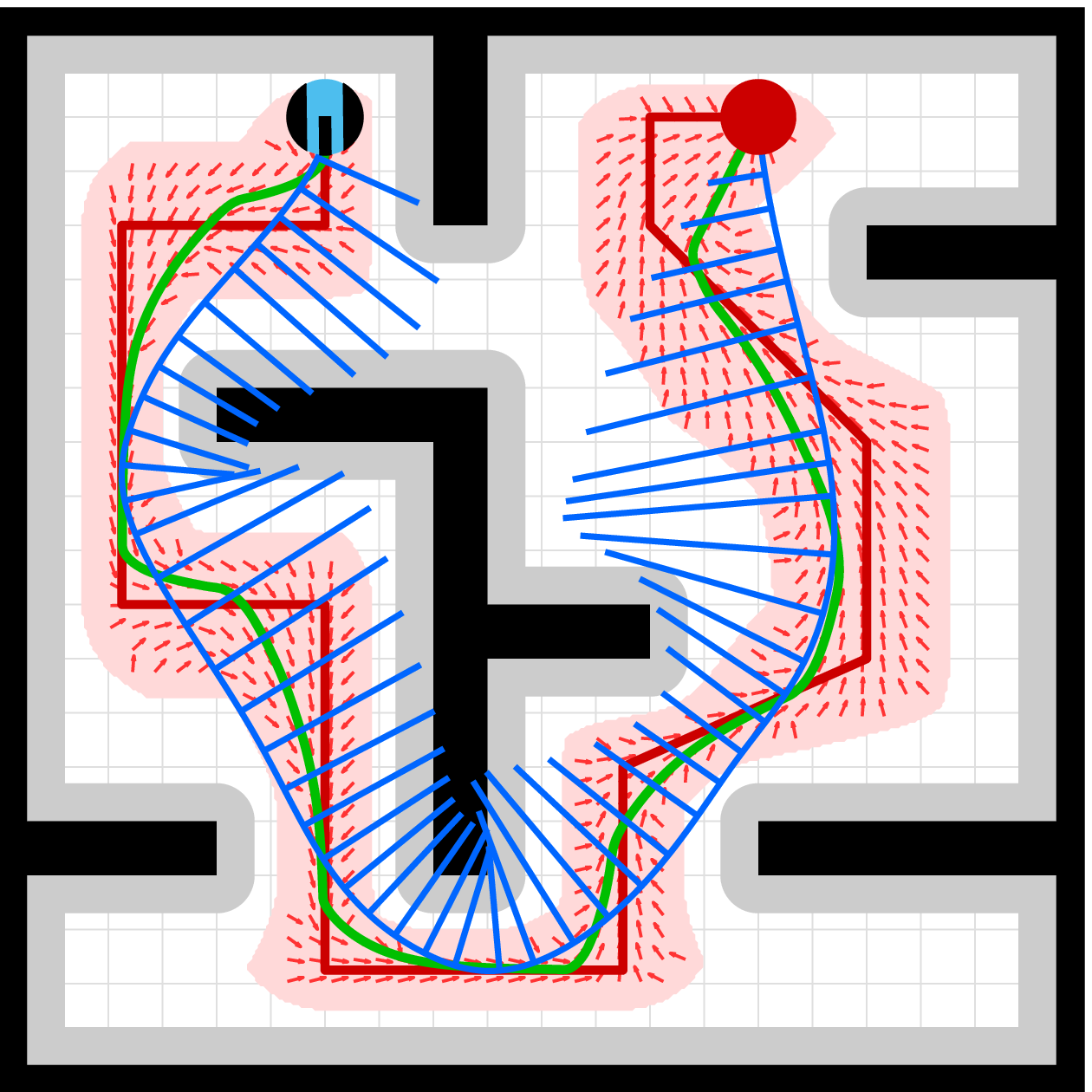}
\\[-1mm]
\footnotesize{(a)} & \footnotesize{(b)} & \footnotesize{(c)} & \footnotesize{(d)}& \footnotesize{(e)}
\end{tabular}
\vspace{-3mm}
\caption{Safe unicycle robot navigation in (top) a corridor environment and (bottom) an office-like cluttered environment using a path-pursuit reference vector field (red arrows) constructed around a piecewise linear reference path (red line) towards a goal position (red circle). The safety of the unicycle motion (blue line) is constantly verified relative to the governor motion (green line) using (a) circular, (b) bounded conic (c) ice-cream-cone-shaped, (d) truncated ice-cream-cone-shaped, (e) forward-simulation-based motion predictions, where the unicycle speed is indicated by bars.}
\label{fig.SafeUnicycleNavigation}
\vspace{-3mm}
\end{figure*}

\section{Numerical Simulations} 
\label{sec.NumericalSimulations}

In this section, we provide numerical simulations%
\footnote{\label{fn.SimulationSettings} For all simulations, we set the linear and angular velocity gains $\lingain = 1$ and $\anggain = 1.5$, the path pursuit planner gain $\gain_\path = 1$, and the governor gain $\govgain = 4$. All simulations are obtained by numerically solving the feedback unicycle-governor dynamics using the \texttt{ode45} function of MATLAB. Please see the accompanying video for the animated robot motion. The open-source code for our MATLAB and ROS implementations is available at {\scriptsize\texttt{https://github.com/core-robotics-research/unicycle\_motion\_ control\_prediction}}.
} 
to demonstrate safe unicycle robot navigation around obstacles using a first-order path pursuit planner as a reference motion planner, where the safety assessment of robot motion is performed based on unicycle feedback motion prediction. 
We also systematically investigate the role of feedback motion range prediction on governed unicycle navigation motion. 
As a baseline ground-truth motion prediction, we use the forward simulation of the closed-loop unicycle navigation dynamics in \refeq{eq.UnicycleGovernorNavigationDynamics} towards the governor position.

\subsection{Path Pursuit Reference Planner}
\label{sec.PathPursuitReferencePlanner}

As a reference motion planner, we consider the ``move-to-projected-path-goal'' navigation policy in \cite{arslan_koditschek_ICRA2017} that constructs a first-order vector field around a given navigation path based on a safe pure pursuit path following approach \cite{coulter_TechReport1992}.

Let $\path : [0,1] \rightarrow \mathring{\freespace}$ be a continuous navigation path inside the free space interior $\mathring{\freespace}$, either generated by a standard path planner \cite{lavalle_PlanningAlgorithms2006} or determined by the user, that connects the start point $\path(0)$ to the end point  $\path(1)=\goal$. 
Accordingly, the first-order ``move-to-projected-path-goal'' law (a.k.a. path pursuit reference planner) $\refplan_{\path}: \refdomain_{\path} \rightarrow \R^{\dimspace}$ is defined  over its positively invariant (Voronoi) domain $\refdomain_{\path}$ \cite{arslan_koditschek_ICRA2017},  which is the generalized Voronoi cell of $\path$ in $\freespace$,
\begin{align}
\refdomain_{\path} \ldf \clist{ \vect{q} \in \freespace \big | \safedist(\vect{q} , \path) \leq \safedist(\vect{q} , \partial \freespace)  },
\end{align}
as 
\begin{align}\label{eq.pathpursuitplanner}
\dot{\govpos} = \refplan_{\path}(\govpos) = - \gain_{\path}(\govpos - \path^{*}(\govpos)),
\end{align}
where $\gain_{\path} > 0$ is a constant gain and the ``projected path goal'', denoted by $\path^*(\govpos)$,  is determined as
\begin{align}
\!\!\path^*(\govpos) \! \ldf\! \path \plist{\! 
\max \plist{\! \Big. \clist{  
\alpha \!\in\! [0,\!1]  \big|  \path(\alpha) \!\in\! \ball(\govpos,\!\safedist(\govpos,\partial\freespace)) 
\!} 
\!} \!\!}.\!\!\!
\end{align}
By construction, for piecewise continuously differentiable navigation paths, the path pursuit planner $\refplan_{\path}$ in \refeq{eq.pathpursuitplanner} is locally Lipschitz continuous and inward-pointing on its domain boundary $\partial\refdomain_{\path}$, and it is asymptotically stable at $\path(1) = \goal$ whose domain of attraction includes the domain $\refdomain_{\path}$ \cite{arslan_koditschek_ICRA2017}.

\begin{figure}[t]
\vspace{2mm}
\centering
\begin{tabular}{@{}c @{\hspace{0.2mm}} c@{}}
\includegraphics[width = 0.5\columnwidth]{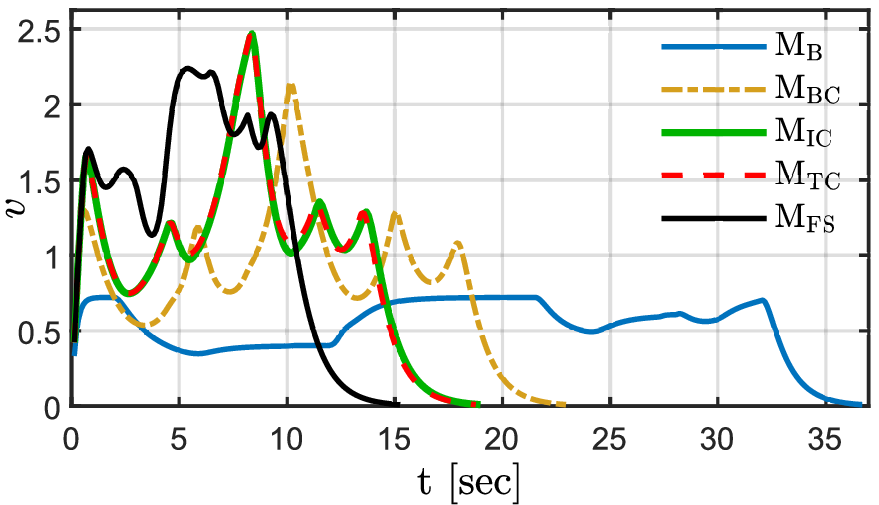} &
\includegraphics[width = 0.5\columnwidth]{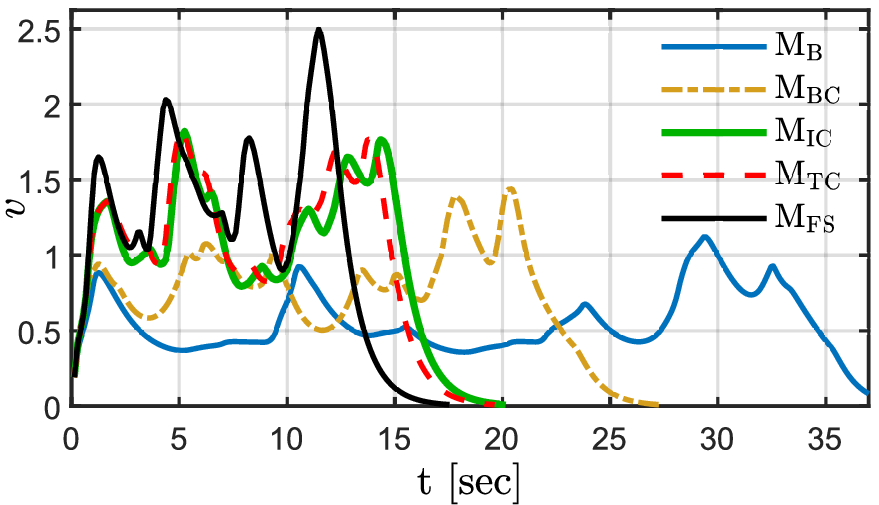}  
\end{tabular}
\vspace{-2mm}
\caption{Unicycle speed profile during safe navigation in (left) a corridor environment and (right) an office-like cluttered environment for different unicycle feedback motion prediction methods: circular motion ball $\mathrm{M}_{\mathrm{B}}$, bounded motion cone $\mathrm{M_{BC}}$, ice-cream motion cone $\mathrm{M_{IC}}$, truncated ice-cream motion cone $\mathrm{M_{TC}}$, forward simulation $\mathrm{M}_{\mathrm{FS}}$.}
\label{fig.SpeedProfile}
\end{figure}

\subsection{Safe Unicycle Navigation in a Corridor Environment}

As a first example, we consider safe unicycle navigation in a corridor environment since safe and fast robot motion control in such tight spaces is challenging \cite{li_ICRA2020}. 
In \reffig{fig.SafeUnicycleNavigation}\,(top), we illustrate the resulting unicycle position trajectories and speeds, where the safety of robot motion relative to the governor is constantly monitored using forward simulation and circular and conic motion prediction methods presented in \refsec{sec.UnicycleFeedbackMotionPrediction}.
As expected, the robot can reach the desired destination of the path pursuit reference planner irrespective of the motion prediction method, but the resulting robot motion significantly differs in terms of robot speed and so travel time, see \reffig{fig.SpeedProfile}.
As seen in \reffig{fig.motion_prediction_demo}, Lyapunov-based circular motion prediction is more conservative in estimating future robot motion because conic motion prediction methods have a stronger dependency on unicycle position and orientation whereas the circular motion prediction method only depends on the unicycle position. 
As a result, conic motion prediction methods always yield faster unicycle robot motion.
We observe that the unicycle robot navigation with the circular motion prediction is more cautious about sideways collisions with corridor walls.
The main difference between the bounded motion cone and the ice-cream motion cone is observed when the robot approaches a turn around the end of a straight corridor, where the relatively conservative bounded cone motion prediction slows the robot down more than the ice-cream-shaped conic motion prediction.
As seen in \reffig{fig.SpeedProfile}, there is no significant difference between the ice-cream motion cone and its truncated version since both of them accurately predict the closed-loop unicycle motion. 
As expected, forward simulation achieves the fastest navigation time and average speed because forward simulation corresponds to the exact feedback motion prediction with a high computational cost.  
We observe in \reffig{fig.SpeedProfile} that compared to the Lyapunov motion prediction, the proposed conic unicycle feedback motion prediction methods can more accurately capture the closed-loop unicycle motion and so can significantly close the performance gap with the exact forward-simulation-based motion prediction.

\subsection{Safe Unicycle Navigation in a Cluttered Environment}

To demonstrate how motion prediction plays a critical role in adapting unicycle motion around complex obstacles, we consider safe robot navigation in an office-like cluttered environment, illustrated in \reffig{fig.SafeUnicycleNavigation} (bottom).
In such an environment, one might naturally expect that the robot slows down while making a turn around obstacles and speeds up if there is a large opening in front of the robot. 
Our numerical studies show that feedback motion prediction significantly influences governed robot motion. 
Conservative (e.g., circular) motion prediction often tends to slow down robot motion because the predicted robot motion cannot be accurately related to the environment. 
As seen in \reffig{fig.SafeUnicycleNavigation}, circular motion prediction is limited in adapting robot motion around obstacles, whereas
conic motion prediction methods allow the robot to leverage available space for faster navigation without compromising safety since conic motion prediction methods can capture robot motion more accurately.
Naturally, the exact forward-simulation-based motion prediction offers further improved adaptation to obstacles at a higher computational cost.  
Overall, accurate motion prediction is crucial for generating safe and fast robot motion around complex (potentially dynamic) obstacles.
To our knowledge and experience, the ice-cream motion cone is currently the best analytic unicycle feedback motion prediction method for safe and fast unicycle robot navigation around obstacles.

\section{Conclusions} 
\label{sec.Conclusions}

In this paper, we introduce novel conic feedback motion prediction methods for bounding the close-loop motion trajectory of the kinematic unicycle robot under a standard forward motion control policy.
The proposed conic motion prediction methods are significantly more accurate in estimating unicycle motion compared to the classical Lyapunov-based circular motion prediction because our conic motion prediction methods depend both on the unicycle position and orientation whereas the circular Lyapunov motion prediction only uses the unicycle position.
Using reference governors, we apply these unicycle motion prediction methods for the safety assessment of robot motion around obstacles for safe robot navigation.
We observe in our numerical studies that the proposed analytic conic unicycle motion prediction performs as well as the forward system simulation at a significantly lower computational cost, which is essential for fast, reactive, and safe robot navigation around obstacles.

Our current work in progress focuses on the sensor-based application of unicycle feedback motion prediction in real hardware experiments, especially for safe robot navigation in unknown dynamic environments \cite{arslan_koditschek_IJRR2019}.
Another promising research direction is the use of unicycle feedback motion prediction in multi-robot navigation and crowd simulation~\cite{vandenberg_lin_manocha_ICRA2008}.
We actively work on the design of new feedback motion prediction methods for nonholonomically constrained robotic systems such as autonomous vehicles and drones.

\bibliographystyle{IEEEtran}
\bibliography{references}

\appendices 

\section{Proofs}
\label{app.Proofs}

\subsection{Proof of \reflem{lem.GlobalStability}}
\label{app.GlobalStability}

\begin{proof}
The result follows from the fact that the unicycle forward motion control turns the robot towards the goal in finite time (\reflem{lem.FiniteTimeGoalAlignment}) and maintains its goal alignment (\reflem{lem.PersistentGoalAlignment}) so that the Euclidean distance to the goal strictly decreases until the robot reaches the goal (\reflem{lem.EuclideanDistance2Goal}).
\end{proof}

\subsection{Proof of \reflem{lem.EuclideanDistance2Goal}}
\label{app.EuclideanDistance2Goal}

\begin{proof}
The result can be verified using the unicycle dynamics in \refeq{eq.UnicycleDynamics} and the unicycle forward motion control in \refeq{eq.forward_unicycle_control} as
\begin{align*}
\frac{\diff}{\diff t} \norm{\pos - \goal}^2 &= 2 \linvel \ovecTsmall{\ort} \! (\pos - \goal) \nonumber
\\
&  = 2 \lingain \max\plist{\!0, \ovecTsmall{\ort} \! (\goal - \pos)\!}  \ovecTsmall{\ort} \!(\pos - \goal)  \nonumber
\\
& = - 2 \lingain \max\plist{\!0, \plist{\!\ovecTsmall{\ort}\! (\goal - \pos)\!}^{\!2}} 
\\
& =   - 2 \lingain \plist{\!\ovecTsmall{\ort}\! (\goal - \pos)\!}^{\!2} \leq 0. \qedhere
\end{align*}
\end{proof}    

\subsection{Proof of \reflem{lem.FiniteTimeGoalAlignment}}
\label{app.FiniteTimeGoalAlignment}

\begin{proof}
For any unicycle pose $(\pos, \ort) \in \R^2 \times [-\pi, \pi)$ with $\ovecTsmall{\ort}(\goal - \pos) \leq 0$ and $\pos \neq \goal$, the linear velocity input is zero, i.e., $\linvel_{\goal}(\pos, \ort) = 0$, and the angular velocity input satisfies $\absval{\angvel_{\goal}(\pos, \ort)} \geq \anggain\frac{\pi}{2}$. 
Moreover, since the unicycle forward motion control changes the unicycle orientation along the geodesic (i.e., shortest) angular path towards the goal, the maximum required angular rotation is $\frac{\pi}{2}$ in order to align the unicycle forward direction towards the goal such that $\ovecTsmall{\ort}(\goal - \pos) > 0$.
Hence, the result holds.     
\end{proof}

\subsection{Proof of \reflem{lem.PersistentGoalAlignment}}
\label{app.PersistentGoalAlignment}

\begin{proof}
If the unicycle orientation is perpendicular to the goal direction, i.e., $\scalebox{0.85}{$\tr{\begin{bmatrix}\cos\ort\\ \sin\ort\end{bmatrix}\!}$} \! (\goal - \pos) = 0$, then the time rate of change of the goal alignment is nonnegative under the unicycle forward motion control, i.e., 
\begin{align}
\frac{\diff}{\diff t}  \ovecTsmall{\ort} \! (\goal - \pos) &= \angvel \nvecTsmall{\ort} \! (\goal - \pos) - \linvel  \nonumber \\
&= \anggain \frac{\pi}{2} \norm{\goal - \pos} \geq 0 \nonumber
\end{align}
since the linear and angular velocity satisfies $\linvel = 0$ and $\angvel = \sgn\plist{\!\scalebox{0.8}{$\tr{\begin{bmatrix}-\sin\ort\\ \cos\ort\end{bmatrix}\!}$} \! (\goal - \pos)\!\!} \anggain\frac{\pi}{2}$ for $\scalebox{0.8}{$\tr{\begin{bmatrix}\cos\ort\\ \sin\ort\end{bmatrix}\!}$} \!\! (\goal - \pos) = 0$, where $\sgn$ denotes the sign function.
Therefore, away from the goal position, the unicycle forward motion control strictly maintains the robot's goal alignment once it strictly points towards the goal, which completes the proof.
\end{proof}

\subsection{Proof of \reflem{lem.PerperdicularGoalAlignmentDistance}}
\label{app.PerpendicularGoalAlignmentDistance}

\begin{proof}
The time rate of change of the squared perpendicular goal alignment distance $\dist_{\goal}^2(\pos, \ort)$ w.r.t. the unicycle dynamics in \refeq{eq.UnicycleDynamics} under the unicycle forward control in \refeq{eq.forward_unicycle_control} is given by
\begin{align}
\frac{\diff}{\diff t} \dist_{\goal}^2(\pos, \ort) 
& =  - 2 \angvel \nvecTsmall{\ort} \! (\goal - \pos) \ovecTsmall{\ort} \! (\goal - \pos)   \nonumber 
\\
& = - 2 \anggain   \alignerr\sin \alignerr  \cos \alignerr \norm{\goal - \pos}^2 \nonumber
\end{align} 
where the angular goal alignment error $\alignerr$ is defined as 
\begin{align}
\alignerr := \atantwo\plist{\!\nvecTsmall{\ort}\! (\goal-\pos), \ovecTsmall{\ort}\! (\goal-\pos)\!} \nonumber
\end{align}
and, by definition, it satisfies for $\pos \neq \goal$ that
\begin{align}
\cos\alignerr =  \ovecTsmall{\ort}\! \frac{\goal-\pos}{\norm{\goal - \pos}} 
\quad \text{and} \quad
\sin \alignerr= \nvecTsmall{\ort}\! \frac{\goal-\pos}{\norm{\goal - \pos}}. \nonumber
\end{align}
Moreover, $\alignerr \sin \alignerr \!\geq\! 0$ and $\cos \alignerr \! \geq \! 0$ for \mbox{$\ovecTsmall{\ort}\!\! (\goal\!-\!\pos) \!\geq\! 0$}, because it follows by definition of $\atantwo$ that $\alignerr \in \blist{-\frac{\pi}{2}, \frac{\pi}{2}}$ when $\scalebox{0.85}{$\tr{\begin{bmatrix}\cos\ort\\ \sin\ort\end{bmatrix}\!}$}\! (\govpos-\pos) \geq 0$. 
Thus, the result holds.
\end{proof}

\subsection{Proof of \refprop{prop.PositiveInclusionCircularMotionPrediction}}
\label{app.PositiveInclusionCircularMotionPrediction}

\begin{proof}
The result directly follows from  \reflem{lem.EuclideanDistance2Goal}  since the Euclidean distance to the goal is decreasing, i.e.,   $\norm{\goal - \pos(t)} \geq \norm{\goal - \pos(t')}$, and so
\begin{align}
 \motionset_{\ctrl_{\goal}, \ball}(\pos(t), \ort(t))  &= \ball(\goal, \norm{\goal - \pos(t)}) \nonumber 
 \\
& \hspace{-5mm} \supseteq \ball(\goal, \norm{\goal - \pos(t')}) = \motionset_{\ctrl_{\goal}, \ball}(\pos(t'), \ort(t')). \qedhere  
\end{align}
\end{proof} 

\subsection{Proof of \refprop{prop.UnboundedConicUnicycleMotionPrediction}}
\label{app.UnboundedConicUnicycleMotionPrediction}

\begin{proof}
Since $\ball(\goal, \norm{\goal- \pos_0}) \subseteq \hplane(\pos_0, \goal)$, it follows from \refprop{prop.CircularUnicycleMotionPrediction} that  $\pos(t) \in \hplane(\pos_0, \goal) $ for all $t \geq 0$. 
Hence, in the rest of the proof, we consider the case $\ovecTsmall{\ort_0}\!\!(\goal - \pos_0) \! \geq \! 0 $.  

If $\ovecTsmall{\ort_0}\!(\goal - \pos_0) \geq  0 $, then \mbox{$\ovecTsmall{\ort(t)}\! (\goal-\pos(t)) \geq 0$}  for all $t \geq 0$ because the unicycle forward motion control maintains a persistent goal alignment (\reflem{lem.PersistentGoalAlignment}).
Moreover, the unicycle forward motion control decreases the perpendicular goal alignment distance at any goal-oriented unicycle pose, i.e.,    $\frac{\diff}{\diff t} \dist_{\goal}(\pos, \ort) \leq 0$ when $\ovecTsmall{\ort}\! (\goal-\pos) \geq 0$ (\reflem{lem.PerperdicularGoalAlignmentDistance}).
Hence, by the definition of the cone, we have
\begin{align}\label{eq.ConeInclusion}
\cone\plist{\pos(t), \goal, \dist_{\goal}(\pos(t), \ort(t))} \subseteq \cone\plist{\pos(t), \goal, \dist_{\goal}(\pos_0, \ort_0)}.
\end{align}
because $\cone(\apex, \base, \radius) \subseteq \cone(\apex, \base, \radius')$ for any $\radius \leq \radius'$.

When  $\ovecTsmall{\ort(t)}\! (\goal-\pos(t)) \geq 0$, by the cone definition, the unicycle velocity $\dot{\pos}(t) = \linvel_{\goal}(\pos(t), \ort(t)\!) \ovectsmall{\ort(t)}$ satisfies 
\begin{align}
\pos(t) + \dot{\pos}(t)  \in \cone\plist{\pos(t), \goal, \dist_{\goal}(\pos(t), \ort(t))} 
\end{align}
because the linear velocity input $\linvel_{\goal}(\pos(t), \ort(t))$ in \refeq{eq.forward_unicycle_control} is nonegative. 
Therefore, we have from \refeq{eq.ConeInclusion} that 
\begin{align}
\pos(t) + \dot{\pos}(t)  \in \cone\plist{\pos(t), \goal, \dist_{\goal}(\pos_0, \ort_0)}
\end{align}
which is to say, the unicycle velocity  $\dot{\pos}(t)$  at position $\pos(t)$ points towards the ball  $\ball(\goal, \dist_{y}(\pos_0, \ort_0))$ that is contained inside the convex cone $\cone(\pos_0, \goal, \dist_{y}(\pos_0, \ort_0))$.
Hence, if the unicycle position trajectory $\pos(t)$ reaches the boundary of the cone $\cone(\pos_0, \goal, \dist_{y}(\pos_0, \ort_0))$, the unicycle velocity $\dot{\pos}(t)$ points inside the cone $\cone(\pos_0, \goal, \dist_{y}(\pos_0, \ort_0))$ and so the unicycle position $\pos(t)$ stays in $\cone(\pos_0, \goal, \dist_{y}(\pos_0, \ort_0))$ for all $t \geq 0$ since $\ball(\goal, \dist_{y}(\pos_0, \ort_0)) \subseteq \cone(\pos_0, \goal, \dist_{y}(\pos_0, \ort_0))$ and $\cone(\pos_0, \goal, \dist_{y}(\pos_0, \ort_0))$ is  convex.
\end{proof}

\subsection{Proof of \refprop{prop.IceCreamUnicycleMotionPrediction}}
\label{app.IceCreamUnicycleMotionPrediction}

\begin{proof}
For any initial unicycle pose $(\pos_0, \ort_0) \in \R^2 \times [-\pi, \pi)$, we have from \refprop{prop.CircularUnicycleMotionPrediction} that $\pos(t) \in \ball(\goal, \norm{\goal - \pos_0})$. 

For any  initial unicycle pose $(\pos_0, \ort_0) \in \R^2 \times [-\pi, \pi)$ with $\ovecT{\ort_0}\!(\goal - \pos_0) \geq 0$, we have from \refprop{prop.UnboundedConicUnicycleMotionPrediction} that the unicycle position trajectory  $\pos(t)$ is contained in the unbounded cone $\cone(\pos_0, \goal, \dist_{\goal}(\pos_0, \ort_0))$.
Note that $\ball(\goal, \dist_{\goal}(\pos_0, \ort_0))$ is the largest ball centered at the goal position $\goal$ and contained in $\cone(\pos_0, \goal, \dist_{\goal}(\pos_0, \ort_0))$.
The removal of $\ball(\goal, \dist_{\goal}(\pos_0, \ort_0))$ divides $\cone(\pos_0, \goal, \dist_{\goal}(\pos_0, \ort_0))$ into a bounded and an unbounded part which are disconnected. 
Hence, since the unicycle forward motion control is globally asymptotically stable at the goal position $\goal$ (\reflem{lem.GlobalStability}), the unicycle position trajectory $\pos(t)$ (starting at $\pos_0$ in the bounded part) eventually enters and stays in  $\ball(\goal, \dist_{\goal}(\pos_0, \ort_0))$ (\reflem{lem.EuclideanDistance2Goal}), without crossing to the other unbounded side of the cone $\cone(\pos_0, \goal, \dist_{\goal}(\pos_0, \ort_0))$. 
Therefore, the unicycle position trajectory $\pos(t)$ is contained in the bounded cone  $\icone(\pos_0, \goal, \dist_{y}(\pos_0, \ort_0)\!) = \conv(\pos_0, \ball(\goal, \dist_{\goal}(\pos_0, \ort_0)))$ for all $t \geq 0$, which completes the proof.
\end{proof}

\subsection{Proof of \refprop{prop.PositiveInclusionIceCreamMotionPrediction}}
\label{app.PositiveInclusionIceCreamMotionPrediction}

\begin{proof}
Since the unicycle forward motion control ensures a persistent goal alignment (\reflem{lem.PersistentGoalAlignment}), the unicycle motion trajectory $(\pos(t), \ort(t))$ starting at $t= 0$ might violate the goal alignment until a finite time $\hat{t}$ (\reflem{lem.FiniteTimeGoalAlignment}) such that
\begin{align}
\ovecTsmall{\ort(t)} (\goal - \pos(t)) < 0 \quad \forall t < \hat{t}\\
\ovecTsmall{\ort(t)} (\goal - \pos(t)) \geq   0 \quad \forall t \geq \hat{t}
\end{align}  
For $0\leq t \leq t'< \hat{t}$, it inherits the positive inclusion property from the circular motion prediction (\refprop{prop.PositiveInclusionCircularMotionPrediction}), i.e.,
\begin{align}
\motionset_{\ctrl_{\goal}, \mathrm{IC}}(\pos(t), \ort(t)) &= \ball(\goal, \norm{\goal - \pos(t)}) \nonumber
\\
&\supseteq \ball(\goal, \norm{\goal - \pos(t')}) = \motionset_{\ctrl_{\goal}, \mathrm{IC}}(\pos(t'), \ort(t')) \nonumber
\end{align}
since the distance to the goal is decreasing (\reflem{lem.EuclideanDistance2Goal}).

For $\hat{t}\leq t \leq t'$, the positive inclusion property   can be observed using \refprop{prop.IceCreamUnicycleMotionPrediction} as
\begin{align}
\motionset_{\ctrl_{\goal}, \mathrm{IC}}(\pos(t), \ort(t)) &= \icone(\pos(t), \goal, \dist_{\goal}(\pos(t), \ort(t))) \\
&\supseteq\icone(\pos(t'), \goal, \dist_{\goal}(\pos(t), \ort(t))) \\
&\supseteq \icone(\pos(t'), \goal, \dist_{\goal}(\pos(t'), \ort(t'))) 
\\
&=\motionset_{\ctrl_{\goal}, \mathrm{IC}}(\pos(t'), \ort(t'))
\end{align}
because $\pos(t') \in \motionset_{\ctrl_{\goal}, \mathrm{IC}}(\pos(t), \ort(t))$ and the perpendicular goal alignment distance is decreasing (\reflem{lem.PerperdicularGoalAlignmentDistance}).

Therefore, the results follows since $\icone(\pos, \ort, \dist_{\goal}(\pos, \ort))= \ball(\goal, \norm{\goal - \pos})$ when $\ovecTsmall{\ort} (\goal - \pos) =   0$.
\end{proof}

\subsection{Proof of \refprop{prop.TruncatedIceCreamMotionCone}}
\label{app.TruncatedIceCreamMotionCone}

\begin{proof}
We provide a sketch of the proof. For any initial condition $(\pos_0, \ort_0) \in \R^2 \times [-\pi, \pi)$ with $\ovecTsmall{\ort_0}(\goal-\pos_0) < 0$, the result directly follows from \refprop{prop.CircularUnicycleMotionPrediction}. 
For any initial pose $(\pos_0, \ort_0)\! \in\! \R^2 \!\times\! [-\pi, \pi)$ with $\ovecTsmall{\ort_0}\!\!(\goal\!-\!\pos_0) \geq 0$, the unicycle forward motion control maintains the goal alignment along the unicycle motion trajectory $(\pos(t), \ort(t))$ for all future times $t \geq 0$ (\reflem{lem.PersistentGoalAlignment}). 
Since the perpendicular alignment distance $\dist_{\goal}(\pos(t), \ort(t))$ is decreasing (\reflem{lem.PerperdicularGoalAlignmentDistance}), the signed perpendicular alignment distance $\dist_{\goal}(\pos(t), \ort(t))$ has  the same sign for all $t \geq 0$. 
Hence, if the unicycle position trajectory $\pos(t)$ crosses the boundary of $\tcone(\pos_0, \goal, \ort_0)$ outside $\ball(\goal, \dist_{\goal(\pos_0, \ort_0)})$, then the unicycle velocity $\dot{x}(t)$ points inside $\tcone(\pos_0, \goal, \ort_0)$ since $\sdist_{\goal}(\pos(t), \ort(t))$ has the same sign with $\sdist_{\goal}(\pos_0, \ort_0)$ and  $\dot{\pos}(t)$ points in the direction of $\ball(\goal, \dist_{\goal}(\pos_0, \ort_0))$ as discussed in the proof of \refprop{prop.UnboundedConicUnicycleMotionPrediction}.
Therefore, due to the global asymptotic stability of the forward motion control (\reflem{lem.GlobalStability}) and the definition of $\tcone(\pos_0, \goal, \ort_0)$, the unicycle position trajectory $\pos(t)$ enters and stays in the positively invariant $\ball(\goal, \dist_{\goal}(\pos_0, \ort_0))$ for all future times (\refprop{prop.CircularUnicycleMotionPrediction}). Thus, the result follows. 
\end{proof}

\subsection{Proof of \refprop{prop.PositiveInclusionTruncatedIceCreamMotionCone}}
\label{app.PositiveInclusionTruncatedIceCreamMotionCone}
%
\begin{proof}
The proof follows the same line of reasoning as the proof of \refprop{prop.PositiveInclusionIceCreamMotionPrediction} where one needs to use the truncated ice-cream cone $\tcone(\pos, \goal, \ort)$ instead of the ice-cream cone $\icone(\pos, \goal, \dist_{\goal}(\pos, \ort))$.
\end{proof}

\subsection{Proof of \refprop{prop.InclusionRelation}}
\label{app.InclusionRelation}

%
\begin{proof}
For $\ovecTsmall{\ort}(\goal - \pos) < 0$, the result simply holds because all motion prediction methods returns $\ball(\goal, \norm{\goal - \pos})$. 
Otherwise, by definition, one has  for $\ovecTsmall{\ort}(\goal - \pos) \geq 0$ that 
\begin{align}
\tcone(\pos, \goal, \ort) &\subseteq \icone(\pos, \goal, \dist_{\goal}(\pos, \ort)) \nonumber 
\\
&\subseteq \cone(\pos, \goal, \dist_{\goal}(\pos, \ort)) \cap \ball(\goal, \norm{\goal - \pos}) \nonumber
\\ 
&\subseteq \ball(\goal, \norm{\goal - \pos}). \nonumber  \qedhere
\end{align}
\end{proof}

\subsection{Proof of \refprop{prop.LipschitzSafetyLevel}}
\label{app.LipschitzSafetyLevel}

%
\begin{proof}
The unicycle feedback motion predictions $\motionset_{\ctrl_{\goal}, \mathrm{B}}, \motionset_{\ctrl_{\goal}, \mathrm{BC}}, \motionset_{\ctrl_{\goal}, \mathrm{IC}}, \motionset_{\ctrl_{\goal}, \mathrm{TC}}$ can be described as a finite collection of circles and triangles whose parameters (e.g., center, radius, and vertices) are a smooth function of the unicycle pose $(\pos, \ort)$ and the goal $\goal$.
Hence, these feedback motion prediction sets can be expressed as a finite collection of an affine transformation of some fixed sets (e.g., the unit ball/ simplex) based on a smooth function of the unicycle pose $(\pos, \ort)$ and the goal $\goal$. 
Therefore, the associated safety level measures are locally Lipschitz continuous since the minimum set distance is Lipschitz continuous under affine transformations (see Lemma 1 in \cite{isleyen_arslan_RAL2022}).   
\end{proof}

\subsection{Proof of \refprop{prop.SafeStableUnicycleGovernorNavigation}}
\label{app.SafeStableUnicycleGovernorNavigation}

\begin{proof}
The result directly follows from the safety and stability of the governor feedback motion design framework \cite{isleyen_arslan_RAL2022} since i) the unicycle forward motion control $\ctrl_{\govpos}$ is Lipschitz continuous almost everywhere and globally asymptotically stable at $\govpos$ (\reflem{lem.GlobalStability}), ii) the feedback motion prediction $\motionrange_{\ctrl_{\govpos}}(\ort, \pos)$ is radially bounded relative to the governor position $\govpos$ (\refprop{prop.InclusionRelation}), and iii) it induces a Lipschitz continuous safety level $\safelevel(\motionrange_{\ctrl_{\govpos}}(\ort, \pos))$ (\refprop{prop.LipschitzSafetyLevel}). 
\end{proof}

\end{document}

%% file: packages.tex
\usepackage{xcolor}
\usepackage{comment}

\usepackage{cite}

\usepackage[mathcal]{euscript}

\usepackage[cmex10]{amsmath}
\usepackage{amssymb}

\usepackage{amsthm} 
  
\usepackage{dsfont} 
\usepackage{mathabx} 

\usepackage{epsfig} 

\usepackage{multirow}
\usepackage{enumerate}

\usepackage[ruled, vlined, linesnumbered]{algorithm2e}

%% file: commands.tex
\newtheoremstyle{plain}
	  {}
	  {}
	  {\itshape}
	  {}
	  {\bfseries}
	  {}
	  {5pt plus 1pt minus 1pt}
	  {}

\newtheoremstyle{definition}
  	  {}
	  {}
	  {\normalfont}
	  {}
	  {\bfseries}
	  {}
	  {5pt plus 1pt minus 1pt}
	  {}
  
\theoremstyle{plain}

\newtheorem{lemma}{Lemma}
\newtheorem{proposition}{Proposition}

\theoremstyle{definition}

\newcommand{\refeq}[1]			{(\ref{#1})} 
\newcommand{\reffig}[1]			{Fig. \ref{#1}} 
\newcommand{\refsec}[1]			{Section \ref{#1}}
\newcommand{\refapp}[1]			{Appendix \ref{#1}}

\newcommand{\refprop}[1]		{Proposition \ref{#1}}

\newcommand{\reflem}[1]			{Lemma \ref{#1}}

\newcommand{\reffn}[1] 		    {\textsuperscript{\ref{#1}}}

%% file: notation.tex
\theoremstyle{plain}

\newcommand{\R}  	{\mathbb{R}} 
\newcommand{\dimspace} 	{d}
\newcommand{\radius} 	{\rho}
\newcommand{\ctr} 	{\vect{c}}
\newcommand{\apex}	{\vect{a}}
\newcommand{\base}  {\vect{b}}
\newcommand{\conv}      {\mathrm{conv}} 
\newcommand{\cone}      {\mathrm{C}} 
\newcommand{\icone} {\widehat{\mathrm{C}}} 
\newcommand{\tcone}{\widecheck{\mathrm{C}}} 
\newcommand{\hplane}{\mathrm{H}} 

\newcommand{\ovect}[1] {\begin{bmatrix}\cos #1\\ \sin #1 \end{bmatrix}}
\newcommand{\ovectsmall}[1] {\scalebox{0.85}{$\ovect{#1}$}}
\newcommand{\ovecT}[1] {\tr{\ovect{#1}\!}}
\newcommand{\ovecTsmall}[1] {\scalebox{0.85}{$\ovecT{#1}$}}
\newcommand{\nvect}[1] {\begin{bmatrix}-\sin #1\\ \cos #1 \end{bmatrix}}

\newcommand{\nvecT}[1] {\tr{\nvect{#1}\!}}
\newcommand{\nvecTsmall}[1] {\scalebox{0.85}{$\nvecT{#1}$}}

\newcommand{\motionset}{\mathcal{M}} 

\newcommand{\freespace}	{\mathcal{F}} 
\newcommand{\workspace}	{\mathcal{W}} 
\newcommand{\obstspace}	{\mathcal{O}} 
\newcommand{\ball}      {\mathrm{B}} 

\newcommand{\path}      {\mathcal{P}}
\newcommand{\sdist} 		{sd} 

\newcommand{\pos} 		{\vect{x}} 			
\newcommand{\ort}	    {\theta}			

\newcommand{\linvel}     {v}  
\newcommand{\angvel}     {\omega} 

\newcommand{\lingain}   	{\kappa_{v}} 			
\newcommand{\anggain}   	{\kappa_{\omega}} 			

\newcommand{\goal}		{\vect{y}} 

\newcommand{\ctrlgoal}	{\vect{y}} 

\newcommand{\globalgoal}      {\vect{x}^*}		

\newcommand{\gain}   	{\kappa} 			
\newcommand{\ctrl}      {\vect{u}} 			

\newcommand{\alignerr}	{\phi} 

\newcommand{\govpos} 	{\vect{y}}

\newcommand{\govgain}   {\kappa_g}

\newcommand{\refplan}  	{\vect{r}} 
\newcommand{\refdomain}	{\mathcal{D}} 

\newcommand{\motionrange}	{\mathcal{M}} 
\newcommand{\safedist}  {\mathrm{dist}}
\newcommand{\safelevel}{\sigma} 

\newcommand{\proj} {\Pi}
\newcommand{\dist} {d} 

\let\originalleft\left
\let\originalright\right
\renewcommand{\left}{\mathopen{}\mathclose\bgroup\originalleft}
\renewcommand{\right}{\aftergroup\egroup\originalright}
	
\newcommand{\plist}[1] 	{\left(#1\right)} 
\newcommand{\blist}[1]	{\left[ #1 \right]} 
\newcommand{\clist}[1]	{\left\{#1\right\}} 

\DeclareMathOperator{\atantwo}{atan2}

\newcommand{\vect}[1]   {\mathrm{#1}}
\newcommand{\mat}[1]    {\mathbf{#1}}
\newcommand{\tr}[1] {{#1}^{\mathrm{T}}} 
\newcommand{\norm}[1]  {\|#1\|}
\newcommand{\absval}[1]{\left|#1 \right|} 

\newcommand{\ldf}   {:=} 

\newcommand{\argmin}{\operatornamewithlimits{arg\ min}} 
\newcommand{\diff} {\mathrm{d}}
\DeclareMathOperator{\sgn}{sgn} 

%% file: manuscript_report.bbl
\begin{thebibliography}{10}
\providecommand{\url}[1]{#1}
\csname url@rmstyle\endcsname
\providecommand{\newblock}{\relax}
\providecommand{\bibinfo}[2]{#2}
\providecommand\BIBentrySTDinterwordspacing{\spaceskip=0pt\relax}
\providecommand\BIBentryALTinterwordstretchfactor{4}
\providecommand\BIBentryALTinterwordspacing{\spaceskip=\fontdimen2\font plus
\BIBentryALTinterwordstretchfactor\fontdimen3\font minus
  \fontdimen4\font\relax}
\providecommand\BIBforeignlanguage[2]{{%
\expandafter\ifx\csname l@#1\endcsname\relax
\typeout{** WARNING: IEEEtran.bst: No hyphenation pattern has been}%
\typeout{** loaded for the language `#1'. Using the pattern for}%
\typeout{** the default language instead.}%
\else
\language=\csname l@#1\endcsname
\fi
#2}}

\bibitem{fiorini_botturi_ISR2008}
P.~Fiorini and D.~Botturi, ``Introducing service robotics to the pharmaceutical
  industry,'' \emph{Intelligent Service Robotics}, vol.~1, no.~4, pp. 267--280,
  2008.

\bibitem{jones_RAM2006}
J.~Jones, ``Robots at the tipping point: the road to irobot roomba,''
  \emph{IEEE Robotics \& Automation Magazine}, vol.~13, no.~1, pp. 76--78,
  2006.

\bibitem{thai_etal_IJMER2022}
N.~H. Thai, T.~T.~K. Ly, H.~Thien, and L.~Q. Dzung, ``Trajectory tracking
  control for differential-drive mobile robot by a variable parameter pid
  controller,'' \emph{International Journal of Mechanical Engineering and
  Robotics Research}, vol.~11, no.~8, 2022.

\bibitem{philippsen_siegwart_ICRA2003}
R.~Philippsen and R.~Siegwart, ``Smooth and efficient obstacle avoidance for a
  tour guide robot,'' in \emph{IEEE International Conference on Robotics and
  Automation}, vol.~1, 2003, pp. 446--451.

\bibitem{prassler_scholz_fiorini_RAM2001}
E.~Prassler, J.~Scholz, and P.~Fiorini, ``A robotics wheelchair for crowded
  public environment,'' \emph{IEEE Robotics \& Automation Magazine}, vol.~8,
  no.~1, pp. 38--45, 2001.

\bibitem{snape_etal_IROS2010}
J.~Snape, J.~van~den Berg, S.~J. Guy, and D.~Manocha, ``Smooth and
  collision-free navigation for multiple robots under differential-drive
  constraints,'' in \emph{IEEE/RSJ International Conference on Intelligent
  Robots and Systems}, 2010, pp. 4584--4589.

\bibitem{chakravarthy_debasish_TSM1998}
A.~Chakravarthy and D.~Ghose, ``Obstacle avoidance in a dynamic environment: a
  collision cone approach,'' \emph{IEEE Trans. Syst. Man Cybern. Part A},
  vol.~28, pp. 562--574, 1998.

\bibitem{fox_burgard_thrun_RAM1997}
D.~Fox, W.~Burgard, and S.~Thrun, ``The dynamic window approach to collision
  avoidance,'' \emph{IEEE Robotics Automation Magazine}, vol.~4, no.~1, pp.
  23--33, 1997.

\bibitem{fiorini_shiller_IJRR1998}
P.~Fiorini and Z.~Shiller, ``Motion planning in dynamic environments using
  velocity obstacles,'' \emph{The International Journal of Robotics Research},
  vol.~17, no.~7, pp. 760--772, 1998.

\bibitem{laumond_MotionPlanning1998}
J.-P. Laumond, S.~Sekhavat, and F.~Lamiraux, ``Guidelines in nonholonomic
  motion planning for mobile robots,'' \emph{Robot motion planning and
  control}, pp. 1--53, 1998.

\bibitem{li_zhang_etal_ROBIO2017}
S.~Li, G.~Zhang, X.~Lei, X.~Yu, H.~Qian, and Y.~Xu, ``Trajectory tracking
  control of a unicycle-type mobile robot with a new planning algorithm,'' in
  \emph{IEEE International Conference on Robotics and Biomimetics}, 2017, pp.
  780--786.

\bibitem{brezak_petrovic_IFAC2011}
M.~Brezak and I.~Petrovi{\'c}, ``Path smoothing using clothoids for
  differential drive mobile robots,'' \emph{International Federation of
  Automatic Control Proceedings Volumes}, vol.~44, no.~1, pp. 1133--1138, 2011.

\bibitem{burridge_rizzi_koditschek_IJRR1999}
R.~R. Burridge, A.~A. Rizzi, and D.~E. Koditschek, ``Sequential composition of
  dynamically dexterous robot behaviors,'' \emph{The International Journal of
  Robotics Research}, vol.~18, no.~6, pp. 535--555, 1999.

\bibitem{blanchini_Automatica1999}
F.~Blanchini, ``Set invariance in control,'' \emph{Automatica}, vol.~35,
  no.~11, pp. 1747 -- 1767, 1999.

\bibitem{pathak_agrawal_TRO2005}
K.~Pathak and S.~Agrawal, ``An integrated path-planning and control approach
  for nonholonomic unicycles using switched local potentials,'' \emph{IEEE
  Transactions on Robotics}, vol.~21, no.~6, pp. 1201--1208, 2005.

\bibitem{conner_choset_rizzi_RSS2006}
\emph{Integrated Planning and Control for Convex-bodied Nonholonomic systems
  using Local Feedback Control Policies}, 2006.

\bibitem{majumdar_tedrake_IJRR2017}
A.~Majumdar and R.~Tedrake, ``Funnel libraries for real-time robust feedback
  motion planning,'' \emph{The International Journal of Robotics Research},
  vol.~36, no.~8, pp. 947--982, 2017.

\bibitem{danielson_berntorp_cairano_weiss_ACC2020}
C.~Danielson, K.~Berntorp, S.~D. Cairano, and A.~Weiss, ``Motion-planning for
  unicycles using the invariant-set motion-planner,'' in \emph{2020 American
  Control Conference}, 2020, pp. 1235--1240.

\bibitem{isleyen_arslan_RAL2022}
A.~{\.I}{\c{s}}leyen, N.~van~de Wouw, and {\"O}.~Arslan, ``From low to high
  order motion planners: Safe robot navigation using motion prediction and
  reference governor,'' \emph{IEEE Robotics and Automation Letters}, vol.~7,
  no.~4, pp. 9715--9722, 2022.

\bibitem{lefevre_vasquez_laugier_ROBOMECH2014}
S.~Lef{\`e}vre, D.~Vasquez, and C.~Laugier, ``A survey on motion prediction and
  risk assessment for intelligent vehicles,'' \emph{ROBOMECH Journal}, vol.~1,
  no.~1, pp. 1--14, 2014.

\bibitem{schubert_richter_wanielik_ICIF2008}
R.~Schubert, E.~Richter, and G.~Wanielik, ``Comparison and evaluation of
  advanced motion models for vehicle tracking,'' in \emph{International
  Conference on Information Fusion}, 2008, pp. 1--6.

\bibitem{schreier_willert_adamy_TITS2016}
M.~Schreier, V.~Willert, and J.~Adamy, ``An integrated approach to
  maneuver-based trajectory prediction and criticality assessment in arbitrary
  road environments,'' \emph{IEEE Transactions on Intelligent Transportation
  Systems}, vol.~17, no.~10, pp. 2751--2766, 2016.

\bibitem{bennewitz_etal_IJRR2005}
M.~Bennewitz, W.~Burgard, G.~Cielniak, and S.~Thrun, ``Learning motion patterns
  of people for compliant robot motion,'' \emph{The International Journal of
  Robotics Research}, vol.~24, no.~1, pp. 31--48, 2005.

\bibitem{arslan_isleyen_arXiv2023}
{\"O}.~Arslan and A.~{\.I}{\c{s}}leyen, ``Vandermonde trajectory bounds for
  linear companion systems,'' \emph{arXiv:2302.10995}, 2023.

\bibitem{althoff_dolan_TR02014}
M.~Althoff and J.~M. Dolan, ``Online verification of automated road vehicles
  using reachability analysis,'' \emph{IEEE Transactions on Robotics}, vol.~30,
  no.~4, pp. 903--918, 2014.

\bibitem{althoff_frehse_girard_ARCRAS2021}
M.~Althoff, G.~Frehse, and A.~Girard, ``Set propagation techniques for
  reachability analysis,'' \emph{Annual Review of Control, Robotics, and
  Autonomous Systems}, vol.~4, pp. 369--395, 2021.

\bibitem{mitchell_HSCC2007}
I.~M. Mitchell, ``Comparing forward and backward reachability as tools for
  safety analysis,'' in \emph{International Workshop on Hybrid Systems:
  Computation and Control}, 2007, pp. 428--443.

\bibitem{bemporad_TAC1998}
A.~Bemporad, ``Reference governor for constrained nonlinear systems,''
  \emph{IEEE Trans. on Automatic Control}, vol.~43, no.~3, pp. 415--419, 1998.

\bibitem{gilbert_kolmanovsky_Automatica2002}
E.~Gilbert and I.~Kolmanovsky, ``Nonlinear tracking control in the presence of
  state and control constraints: a generalized reference governor,''
  \emph{Automatica}, vol.~38, no.~12, pp. 2063 -- 2073, 2002.

\bibitem{garone_nicotra_TAC2015}
E.~Garone and M.~M. Nicotra, ``Explicit reference governor for constrained
  nonlinear systems,'' \emph{IEEE Transactions on Automatic Control}, vol.~61,
  no.~5, pp. 1379--1384, 2015.

\bibitem{arslan_koditschek_ICRA2017}
{\"O}.~Arslan and D.~E. Koditschek, ``Smooth extensions of feedback motion
  planners via reference governors,'' in \emph{IEEE Int. Conf. on Robotics and
  Automation}, 2017, pp. 4414--4421.

\bibitem{li_ICRA2020}
Z.~Li, {\"O}.~Arslan, and N.~Atanasov, ``Fast and safe path-following control
  using a state-dependent directional metric,'' in \emph{IEEE Int. Conf. on
  Robotics and Automation}, 2020, pp. 6176--6182.

\bibitem{li_2020}
Z.~Li, T.~Duong, and N.~Atanasov, ``Safe robot navigation in cluttered
  environments using invariant ellipsoids and a reference governor,''
  \emph{arXiv 2005.06694}, 2020.

\bibitem{astolfi_JDSMC1999}
A.~Astolfi, ``{Exponential Stabilization of a Wheeled Mobile Robot Via
  Discontinuous Control},'' \emph{Journal of Dynamic Systems, Measurement, and
  Control}, vol. 121, no.~1, pp. 121--126, 1999.

\bibitem{arslan_koditschek_IJRR2019}
{\"O}.~Arslan and D.~E. Koditschek, ``Sensor-based reactive navigation in
  unknown convex sphere worlds,'' \emph{The International Journal of Robotics
  Research}, vol.~38, no. 2-3, pp. 196--223, 2019.

\bibitem{Brockett_DGCT1983}
R.~W. Brockett, ``Asymptotic stability and feedback stabilization,'' in
  \emph{Differential Geometric Control Theory}, 1983, pp. 181--191.

\bibitem{lee_etal_IROS2000}
S.-O. Lee, Y.-J. Cho, M.~Hwang-Bo, B.-J. You, and S.-R. Oh, ``A stable
  target-tracking control for unicycle mobile robots,'' in \emph{IEEE/RSJ
  International Conference on Intelligent Robots and Systems}, 2000, pp.
  1822--1827 vol.3.

\bibitem{khalil_NonlinearSystems2001}
H.~K. Khalil, \emph{Nonlinear Systems}.\hskip 1em plus 0.5em minus 0.4em\relax
  Prentice Hall, 2001.

\bibitem{lavalle_PlanningAlgorithms2006}
S.~M. LaValle, \emph{Planning Algorithms}.\hskip 1em plus 0.5em minus
  0.4em\relax Cambridge Univ. Press, 2006.

\bibitem{choset_etal_PrinciplesOfRobotMotion2005}
H.~M. Choset, K.~M. Lynch, S.~Hutchinson, G.~Kantor, W.~Burgard, L.~Kavraki,
  S.~Thrun, and R.~C. Arkin, \emph{Principles of Robot Motion: Theory,
  Algorithms, and Implementations}.\hskip 1em plus 0.5em minus 0.4em\relax MIT
  Press, 2005.

\bibitem{coulter_TechReport1992}
R.~C. Coulter, ``Implementation of the pure pursuit path tracking algorithm,''
  Carnegie Mellon University, Tech. Rep., 1992.

\bibitem{vandenberg_lin_manocha_ICRA2008}
J.~van~den Berg, M.~Lin, and D.~Manocha, ``Reciprocal velocity obstacles for
  real-time multi-agent navigation,'' in \emph{IEEE International Conference on
  Robotics and Automation}, 2008, pp. 1928--1935.

\end{thebibliography}
